\documentclass{article}

\usepackage{microtype}
\usepackage{graphicx}
\usepackage{subfigure}
\usepackage{booktabs} 

\usepackage{amsmath}
\usepackage{amsthm}
\usepackage{amssymb}
\usepackage{mathtools}
\usepackage{mathrsfs}
\mathtoolsset{showonlyrefs}

\usepackage{textcomp}

\usepackage[utf8]{inputenc} 
\usepackage[T1]{fontenc}    
\usepackage{hyperref}       
\usepackage{url}            
\usepackage{booktabs}       
\usepackage{amsfonts}       
\usepackage{nicefrac}       
\usepackage{microtype}      

\usepackage{xcolor}

\usepackage[accepted]{icml2020}

\icmltitlerunning{Asynchronous Coagent Networks}

\usepackage{bbm}

\usepackage{sidecap} 

\newtheorem{thm}{Theorem}[]

\newtheorem{lemma}{Lemma}[]
\newtheorem{prop}{Property}[]

\newtheorem{cor}{Corollary}[]

\usepackage{enumitem}

\usepackage[english]{babel}
\usepackage{blindtext}

\usepackage{hyperref}

\newtheorem{assumption}{Assumption}

\newcommand{\pre}{\text{pre}}
\newcommand{\post}{\text{post}}
\newcommand{\all}{\text{all}}
\newcommand{\ith}{i^\text{th}}

\begin{document}

\twocolumn[
\icmltitle{Asynchronous Coagent Networks}



\icmlsetsymbol{equal}{*}

\begin{icmlauthorlist}
\icmlauthor{James E. Kostas}{equal,um}
\icmlauthor{Chris Nota}{equal,um}
\icmlauthor{Philip S. Thomas}{um}
\end{icmlauthorlist}

\icmlaffiliation{um}{College of Information and Computer Sciences, University of Massachusetts, Amherst, MA, USA}

\icmlcorrespondingauthor{James E. Kostas}{jekostas@umass.edu}

\icmlkeywords{Machine Learning, ICML}

\vskip 0.3in
]



\printAffiliationsAndNotice{\icmlEqualContribution} 

\begin{abstract}
\emph{Coagent policy gradient algorithms} (CPGAs) are reinforcement learning algorithms for training a class of stochastic neural networks called \emph{coagent networks}.
In this work, we prove that CPGAs converge to locally optimal policies.
Additionally, we extend prior theory to encompass \emph{asynchronous} and \emph{recurrent} coagent networks.
These extensions facilitate the straightforward design and analysis of hierarchical reinforcement learning algorithms like the option-critic, and eliminate the need for complex derivations of customized learning rules for these algorithms.
\end{abstract}
\section{Introduction}
\emph{Reinforcement learning} (RL) policies are often represented by \emph{stochastic neural networks} (SNNs).
SNNs are networks where the outputs of some units are not deterministic functions of the units' inputs.
Several classes of algorithms from various branches of RL research, such as those using options \citep{sutton1999between} or hierarchical architectures \citep{bacon2017option}, can be formulated as using SNN policies (see Section \ref{sec:relatedwork} for more examples).
In this paper we study the problem of deriving learning rules for RL agents with SNN policies.
%

\emph{Coagent networks} are one formulation of SNN policies for RL agents \citep{Thomas2011}. 
Coagent networks are comprised of conjugate agents, or \emph{coagents}; each coagent is an RL algorithm learning and acting cooperatively with the other coagents in its network.
In this paper, we focus specifically on the case where each coagent is a policy gradient RL algorithm, and call the resulting algorithms \emph{coagent policy gradient algorithms} (CPGAs).
Intuitively, CPGAs cause each individual coagent to view the other coagents as part of the environment.
That is, individual coagents learn and interact with the combination of both the environment and the other coagents as if this combination was a single environment.


Typically, algorithm designers using SNNs must create specialized training algorithms for their architectures and prove the correctness of these algorithms.
Coagent policy gradient algorithms (CPGAs) provide an alternative: they allow algorithm designers to easily derive stochastic gradient update rules for the wide variety of policy architectures that can be represented as SNNs.
To analyze a given policy architecture (SNN), one must simply identify the inputs and outputs of each coagent in the network.
The theory in this paper then immediately provides a policy gradient (stochastic gradient ascent) learning rule for each coagent, providing a simple mechanism for obtaining convergent update rules for complex policy architectures.
This process is applicable to SNNs across several branches of RL research.

This paper also extends that theory and the theory in prior work to encompass \emph{recurrent} and \emph{asynchronous} coagent networks.
A network is \emph{recurrent} if it contains cycles.
A network is \emph{asynchronous} if units in the neural network do not execute simultaneously or at the same rate.
The latter property allows distributed implementations of large neural networks to operate asynchronously.
Additionally, a coagent network's capacity for \emph{temporal abstraction} (learning, reasoning, and acting across different scales of time and task) may be enhanced, not just through the network topology, but by designing networks where different coagents execute at different rates.
Furthermore, these extensions facilitate the straightforward design and analysis of hierarchical RL algorithms like the option-critic.

The contributions of this paper are: \textbf{1)} a complete and formal proof of a key CPGA result on which this paper relies (prior work provides an informal and incomplete proof), \textbf{2)} a generalization of the CPGA framework to handle asynchronous recurrent networks, \textbf{3)} empirical support of our theoretical claims regarding the
gradients of asynchronous CPGAs, 
and \textbf{4)} a proof that asynchronous CPGAs generalize the option-critic framework, which serves as a demonstration of how CPGAs eliminate the need for the derivation of custom learning rules for architectures like the option-critic.
Our simple mechanistic approach to gradient derivation for the option-critic is a clear example of the usefulness of the coagent framework to any researcher or algorithm designer creating or analyzing stochastic networks for RL.

\section{Related Work}
\label{sec:relatedwork}
\citet{klopf1982hedonistic} theorized that traditional models of classical and operant conditioning could be explained by modeling biological neurons as \emph{hedonistic}, that is, seeking excitation and avoiding inhibition.
The ideas motivating coagent networks bear a deep resemblance to Klopf's proposal.

Stochastic neural networks have applications dating back at least to Marvin Minsky's \emph{stochastic neural analog reinforcement calculator}, built in 1951 \citep{russell2016artificial}.
Research of stochastic learning automata continued this work \citep{narendra1989learning}; one notable example is the \emph{adaptive reward-penalty} learning rule for training stochastic networks \citep{barto1985learning}.
Similarly, \citet{williams1992simple} proposed the well-known REINFORCE algorithm with the intent of training stochastic networks. 
Since then, REINFORCE has primarily been applied to deterministic networks.
However, \citet{Thomas2011b} proposed CPGAs for RL, building on the original intent of \citet{williams1992simple}.

Since their introduction, CPGAs have proven to be a practical tool for defining and improving RL agents.
CPGAs have been used to discover ``motor primitives'' in simulated robotic control tasks and to solve RL problems with high-dimensional action spaces \citep{Thomas2012}.
They are the RL precursor to the more general stochastic computation graphs.
%

The formalism of \emph{stochastic computation graphs} was proposed to describe networks with a mixture of stochastic and deterministic nodes, with applications to supervised learning, unsupervised learning, and RL \citep{schulman2015gradient}.
Several recently proposed approaches fit into the formalism of stochastic networks, but the relationship has frequently gone unnoticed.
One notable example is the \emph{option-critic} architecture \citep{bacon2017option}.
The option-critic provides a framework for learning \emph{options} \citep{sutton1999between}, a type of high-level and temporally extended action, and how to choose between options. Below, we show that this framework is a special case of coagent networks.  
Subsequent work, such as the \emph{double actor-critic architecture for learning options} \citep{NIPS2019_8475} and the \emph{hierarchical option-critic} \citep{NIPS2018_8243} also fit within and can be informed by the coagent framework.

%

The Horde architecture \citep{Sutton2011} is similar to the coagent framework in that it consists of independent RL agents working cooperatively.
However, the Horde architecture does not cause the collection of all agents to follow the gradient of the expected return with respect to each agent's parameters---the property that allows us to show the convergence of CPGAs.

%

CPGAs can be viewed as a special case of \emph{multi-agent reinforcement learning} (MARL), the application of RL in settings where multiple agents exist and interact.
However, MARL differs from CPGAs in that MARL agents typically have separate manifestations within the environment; additionally, the objectives of the MARL agents may or may not be aligned.
Working within the MARL framework, researchers have proposed a variety of principled algorithms for cooperative multi-agent learning \citep{guestrin2002coordinated, zhang2007conditional, liu2014optimal}.
An overview of MARL is given by \citet{bucsoniu2010multi}.

\section{Background}
\label{sec:background}
We consider an MDP, $M=(\mathcal S, \mathcal A, \mathcal R, P, R, d_0,\gamma)$, where 
$\mathcal S$ is the finite set of possible \emph{states}, 
$\mathcal A$ is the finite set of possible \emph{actions},
and $\mathcal R$ is the finite set of possible \emph{rewards} (although this work extends to MDPs where these sets are infinite and uncountable, the assumption that they are finite simplifies notation).
Let $t \in \{0,1,2,\dots\}$ denote the time step. 
$S_t$, $A_t$, and $R_t$ are the state, action, and reward at time $t$, and are random variables that take values in $\mathcal S$, $\mathcal A$, and $\mathcal R$, respectively. 
$P{:} \mathcal S {\times} \mathcal A {\times} \mathcal S {\to} [0, 1]$ is the \emph{transition function}, given by $P(s,a,s')\coloneqq\Pr(S_{t+1}{=}s'|S_t{=}s,A_t{=}a)$.
$R{:} \mathcal S {\times} \mathcal A {\times} \mathcal S {\times} \mathcal R {\to} [0, 1]$ is the \emph{reward distribution}, given by $R(s,a,s',r)\coloneqq\Pr(R_t{=}r|S_t{=}s,A_t{=}a,S_{t+1}{=}s')$.
The \emph{initial state distribution}, $d_0{:}\mathcal S \to [0, 1]$, is given by $d_0(s)\coloneqq\Pr(S_0{=}s)$.
The reward discount parameter is $\gamma \in [0,1]$.
An \emph{episode} is a sequence of states, actions, and rewards, starting from $t {=} 0$ and continuing indefinitely.
We assume that the discounted sum of rewards over an episode is finite.

A policy, $\pi:\mathcal S \times \mathcal A \to [0,1]$, is a stochastic method of selecting actions, such that $\pi(s,a)\coloneqq\Pr(A_t{=}a|S_t{=}s)$.
A \emph{parameterized policy} is a policy that takes 
a parameter vector $\theta \in \mathbb R^n$.
Different parameter vectors result in 
different policies. 
More formally, we redefine the symbol $\pi$ to denote a parameterized policy, $\pi:\mathcal S \times \mathcal A \times \mathbb R^n \to [0,1]$, such that for all $\theta \in \mathbb R^n$, $\pi(\cdot,\cdot,\theta)$ is a policy.
We assume that $\partial \pi(s,a, \theta) / \partial \theta$ exists for all $s \in \mathcal{S}, a \in \mathcal{A}, \text{ and } \theta \in \mathbb R^n$.
An agent's goal is typically to find a policy that maximizes the \emph{objective function}
$
J(\pi)\coloneqq\mathbf{E}\left [ \sum_{t=0}^\infty \gamma^t R_t \middle| \pi \right ],
$
where conditioning on $\pi$ denotes that, for all $t$, $A_{t} \sim \pi(S_{t},\cdot)$.
The state-value function, $v^\pi:\mathcal S \to \mathbb R$, is defined as $ v^\pi(s)\coloneqq\mathbf{E}\left [\sum_{k=0}^\infty \gamma^k R_{t+k}\middle |S_t{=}s, \pi\right ].
$
The discounted return, $G_t$, is defined as $G_t \coloneqq \sum_{k=0}^\infty\gamma^{k} R_{t+k}$.
We 
denote the objective function for a policy that has parameters $\theta$ as $J(\theta)$, and condition probabilities on $\theta$ to denote that the parameterized policy uses parameter vector $\theta$.

A \emph{coagent network} is a parameterized policy that consists of an acyclic network of nodes (coagents), which do not share parameters.
Each coagent can have several inputs that may include the state at time $t$, a noisy and incomplete observation of the state at time $t$, and/or the outputs of other coagents.
When considering the $\ith$ coagent, $\theta$ can be partitioned into two vectors, $\theta_i \in \mathbb R^{n_i}$ (the parameters used by the $\ith$ coagent) and $\bar \theta_i \in \mathbb R^{n-n_i}$ (the parameters used by all other coagents). 
From the point of view of the $i^\text{th}$ coagent, $A_t$ is produced from $S_t$ in three stages:
execution of the nodes prior to the $i^\text{th}$ coagent (nodes whose outputs are required to compute the input to the $i^\text{th}$ coagent), execution of the $i^\text{th}$ coagent, and execution of the remaining nodes in the network to produce the final action. 
This process is depicted graphically in Figure \ref{fig:CoagentDiagram_cropped} and described in detail below.

First, we define a parameterized distribution $\pi^\pre_i(S_t, \cdot, \bar \theta_i)$ to capture how the previous coagents in the network produce their outputs given the current state. 
The output of the previous coagents is a random variable, which we denote by $U^\pre_t$, and which takes continuous and/or discrete values in some set $\mathcal U^\pre$. $U^\pre_t$ is sampled from the distribution $\pi^\pre_i(S_t, \cdot, \bar \theta_i)$.
Next, the $\ith$ coagent takes $S_t$ and $U^\pre_t$ as input. We denote this input, $(S_t, U_t^\pre)$, as $X_t$ (or $X^i_t$ if it is not unambiguously referring to the $\ith$ coagent), and refer to $X_t$ as the \emph{local state}.
Given this input, the coagent produces the output $U^i_t$ (below, when it is unambiguously referring to the output of the $\ith$ coagent, we make the $i$ implicit and write $U_t$).
The conditional distribution of $U^i_t$ is given by the $\ith$ coagent's policy, 
$\pi_i(X_t,\cdot, \theta_i)$. 
Although we allow the $\ith$ coagent's output to depend directly on $S_t$, it may be parameterized to only depend on $U_t^\text{pre}$.
Finally, $A_t$ is sampled according to a distribution $\pi_i^\post(X_t,U^i_t, \cdot, \bar \theta_i)$, which captures how the subsequent coagents in the network produce $A_t$.
Below, we sometimes make $\bar \theta_i$ and $\theta_i$ implicit and write the three policy functions as $\pi^\pre_i(S_t, \cdot)$, $\pi_i(X_t,\cdot)$, and $\pi_i^\post(X_t,U^i_t, \cdot)$.
Figure \ref{fig:Causal} depicts the setup that we have described and makes relevant independence properties clear.
%
\begin{figure}
    \centering
    \includegraphics[width=.40\textwidth]{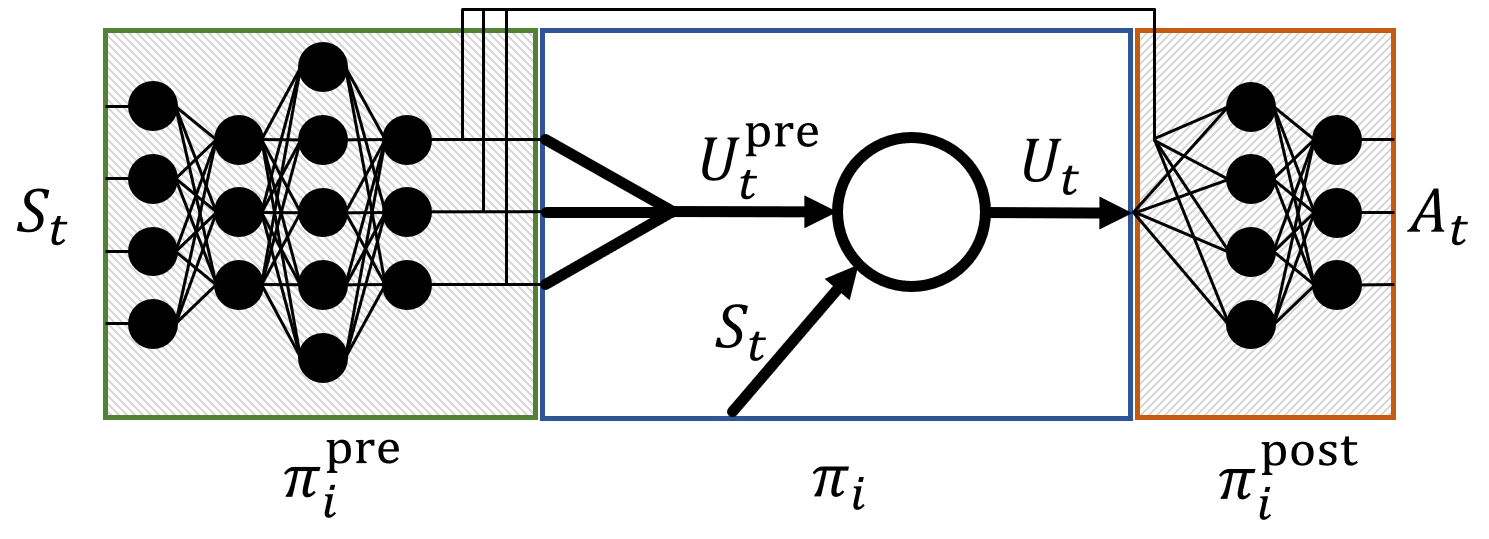}
    \caption{
       Diagram of the three-step process for action generation for a fully connected feedforward network (we do not require the network to have this structure).
       The circle in the middle denotes the $i^\text{th}$ coagent. First, preceding nodes are executed to compute the inputs to this coagent. Second, the coagent uses these inputs to produce its output, $U_t$. Third, the remainder of the network is executed to produce an action.
       For each coagent, all inputs are optional.  That is, our approach encompasses architectures where $S_t$ and/or components of $U^\pre_t$ are not provided to some coagents.
    }
    \label{fig:CoagentDiagram_cropped}
\end{figure}
\section{The Coagent Policy Gradient Theorem}
\label{sec:acorrect}
Consider what would happen if the $\ith$ coagent ignored all of the complexity in this problem setup and simply learned and interacted with the combination of both the environment and the other coagents as if this combination was a single environment.
From the $\ith$ coagent's point of view, the actions would be $U_t$, the rewards would remain $R_t$, and the state would be $X_t$ (that is, $S_t$ and $U_t^\text{pre}$). Note that the coagent may ignore components of this local state, such as the $\mathcal S$ component.
Each coagent could naively implement an unbiased policy gradient algorithm, like REINFORCE \citep{williams1992simple}, based only on these inputs and outputs.
We refer to the expected update in this setting as the \emph{local policy gradient}, $\Delta_i$, for the $\ith$ coagent. 
Formally, the local policy gradient of the $\ith$ coagent is:
$$
    \Delta_i(\theta_i) \coloneqq \mathbf{E} \left[ \sum_{t=0}^\infty \gamma^t G_t \frac{\partial \ln\left ( \pi_i\left ( X_t, U_t, \theta_i \right ) \right )}{\partial \theta_i} \middle | \theta \right].
$$
The local policy gradient should not be confused with $\nabla J(\theta)$, which we call the \emph{global policy gradient}, or with $\nabla J(\theta)$'s $\ith$ component, $[\partial J(\theta)/\partial \theta_i]^\intercal$
(notice that $\nabla J(\theta) = \big [ \frac{\partial J(\theta)}{\partial \theta_1}^\intercal , \frac{\partial J(\theta)}{\partial \theta_2}^\intercal,\dotsc,\frac{\partial J(\theta)}{\partial \theta_m}^\intercal\big ]^\intercal$, where $m$ is the number of coagents).
Unlike a network following $\nabla J(\theta)$ or a coagent following its corresponding component of $\nabla J(\theta)$, a coagent following a local policy gradient faces a non-stationary sequence of partially-observable MDPs as the other coagents (part of its environment) update and learn.
One could naively design an algorithm that uses this local policy gradient and simply hope for good results, but without theoretical analysis, this hope would not be justified.

%
The \emph{coagent policy gradient theorem} (CPGT) justifies this approach: If $\theta$ is fixed and all coagents update their parameters following unbiased estimates, $\widehat \Delta_i(\theta_i)$, of their local policy gradients, then the entire network will follow an unbiased estimator of $\nabla J(\theta)$. 
For example, if every coagent performs the following update simultaneously at the end of each episode, then the entire network will be performing stochastic gradient ascent on $J$ (without using backpropagation): 
\begin{align}
    \theta_i \gets \theta_i + \alpha \sum_{t=0}^\infty \gamma^t G_t \left(\frac{\partial \ln\left ( \pi_i\left ( X_t, U_t, \theta_i \right ) \right )}{\partial \theta_i}\right). 
\end{align} 
In practice, one would use a more sophisticated policy gradient algorithm than this simple variant of REINFORCE.

Although \citet{Thomas2011} present the CPGT in their Theorem 3, the provided proof is lacking in two ways.
First, it is not general enough for our purposes because it only considers networks with two coagents. 
Second, it is missing a crucial step. 
They define a new MDP, the CoMDP, which models the environment faced by a coagent
%
%
, but they do not show that this definition accurately describes the environment that the coagent faces. 
Without this step, \citet{Thomas2011} have shown that there is a new MDP for which the policy gradient is a component of $\nabla J(\theta)$, but not that this MDP has any relation to the coagent network. 
\begin{figure}
    \centering
    \includegraphics[width=.45\textwidth]{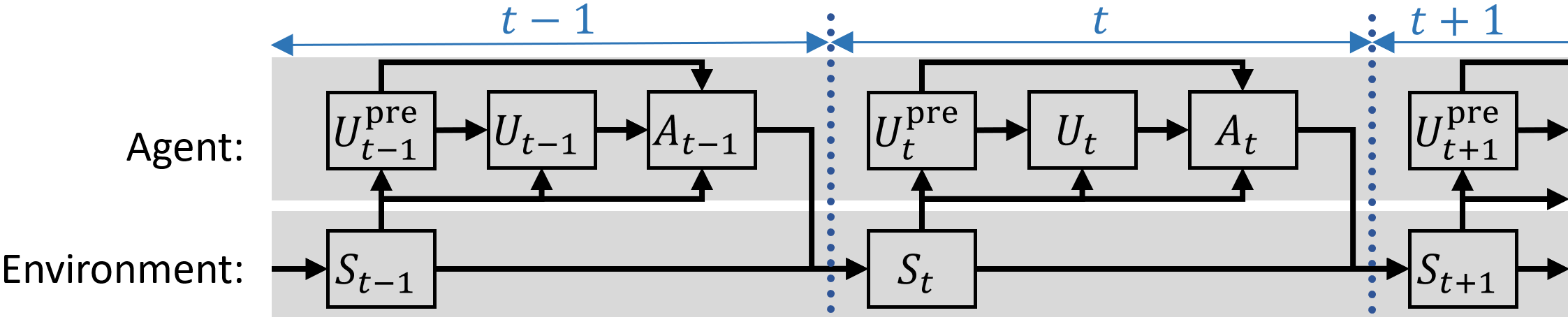}
    \caption{Bayesian network depicting the relationships of relevant random variables. 
    }
    \label{fig:Causal}
\end{figure}
\subsection{The Conjugate Markov Decision Process (CoMDP)}
In order to reason about the local policy gradient, we begin by modeling the $\ith$ coagent's environment as an MDP, called the CoMDP.
Given $M$, $i$, $\pi^\pre_i$, $\pi^\post_i$, and $\bar \theta_i$, we define a corresponding CoMDP, $M^i$, as $M^i\coloneqq(\mathcal X^i,\mathcal U^i, \mathcal R^i, P^i, R^i, d_0^i, \gamma_i)$. Although our proof of the CPGT relies heavily on these definitions, due to space limitations the formal definitions for each of these terms are provided in Section \ref{sec:complete_comdp} of the supplementary material.
A brief summary of definitions follows.

We write $\Tilde{X}_t^i$, $\Tilde{U}_t^i$, and $\Tilde{R}_t^i$ to denote the state, action, and reward of $M^i$ at time $t$. These \emph{random variables} in the CoMDP are written with tildes to provide a visual distinction between terms from the CoMDP and original MDP.  Additionally, when it is clear that we are referring to the $\ith$ CoMDP, we often make $i$ implicit (for example, we might write $\Tilde{X}_t^i$ as $\Tilde{X}_t$).

The input (analogous to the state) to the $\ith$ coagent is in the set $\mathcal X^i \coloneqq \mathcal S \times \mathcal U^\pre_i$.  For $x \in \mathcal{X}$, we denote the $\mathcal{S}$ component as $x.s$ and the $\mathcal U^\pre$ component as $x.u_\pre$.  We also sometimes denote an $x \in \mathcal{X}^i$ as $(x.s, x.u_\pre)$. $\mathcal U^i$ is an arbitrary set that denotes the output of the $\ith$ coagent. $\mathcal R^i\coloneqq\mathcal R$ and $\gamma_i\coloneqq \gamma$ represent the reward set and the discount factor, respectively. 
We denote the transition function as $P^i(x,u,x')$, the reward function as $R^i(x,u,x',r)$, and initial state distribution as $d^i_0(x)$.
We write $J_i(\theta_i)$ to denote the objective function of $M^i$. 
\subsection{The CoMDP Models the Coagent's Environment}

Here we show that our definition of the CoMDP correctly models the coagent's environment. 
We do so by presenting a series of properties and lemmas that each establish different components of the relationship between the CoMDP and the environment faced by a coagent. 
The proofs of these properties and theorems are provided in Section \ref{sec:complete_CPGT_proofs} of the supplementary material. 

In Properties \ref{correct_di0} and \ref{correct_s0}, by manipulating the definitions of $d_0^i$ and $\pi_i^\text{pre}$, we show that $d_0^i$ and the distribution of $\Tilde X_0.s$ capture the distribution of the inputs to the $\ith$ coagent.
\begin{prop}
$
\forall x \in \mathcal X,
d_0^i(x)=\Pr(S_0{=}x.s, U_0^\text{\emph{pre}}{=}x.u_\text{\emph{pre}}).
$
\end{prop}
\begin{prop}
\label{correct_s0}
For all $s \in \mathcal S$,
$
\Pr(\Tilde X_0.s{=}s)=d_0(s).
$
\end{prop}

In Property \ref{correct_transition}, we show that $P^i$ captures the distributions of the inputs that the $\ith$ coagent will see given the input at the previous step and the output that it selected.

\begin{prop}
For all $x\in \mathcal X,$ $x' \in \mathcal X$, and $u \in \mathcal U$,
$
P^i(x,u,x') = \Pr(S_{t+1}{=}x'.s,U^\text{\emph{pre}}_{t+1}{=}x'.u_\text{\emph{pre}}\\
|S_{t}{=}x.s,U^\text{\emph{pre}}_{t}{=}x.u_\text{\emph{pre}}, U_t {=} u).
$
\end{prop}

In Property \ref{correct_reward},
we show that $R^i$ captures the distribution of the rewards that the $\ith$ coagent receives given the output that it selected and the inputs at the current and next steps.
\begin{prop}
For all $x\in \mathcal X$, $x' \in \mathcal X$, $u \in \mathcal U$, and $r \in \mathcal{R}$,
$
 R^i(x,u,x',r) = \Pr(R_t {=} r
| S_{t}{=}x.s,U^\text{\emph{pre}}_{t}{=}x.u_\text{\emph{pre}}, U_t {=} u,
\\S_{t+1}{=}x'.s, U^\text{\emph{pre}}_{t+1}{=}x'.u_\text{\emph{pre}}).
$
\end{prop}
In Properties \ref{correct_x} and \ref{correct_s}, we show that the distributions of $\Tilde{X}$ and $\Tilde{X}_t.s$ capture the distribution of inputs to the $\ith$ coagent.

\begin{prop}
For all $s {\in} \mathcal{S}$ and $u_\text{\emph{pre}} {\in} \mathcal U^\text{\emph{pre}}_i$,

$
\Pr(\Tilde{X}_t {=} (s, u_\text{\emph{pre}})) = \Pr(S_t {=} s, U^\text{\emph{pre}}_t {=} u_\text{\emph{pre}}).
$
\end{prop}
\begin{prop}
For all $s \in \mathcal{S}$,
$
\Pr(\Tilde{X}_t.s {=} s) = \Pr(S_t {=} s).
$
\end{prop}

In Property \ref{correct_pi_pre}, we show that the distribution of $\Tilde{X}_t.u_\text{\emph{pre}}$ given $\Tilde{X}_t.s$ captures the distribution $\pi^\text{\pre}_i$.
\begin{prop}
For all $s \in \mathcal{S}$ and $u_\text{\emph{pre}} \in \mathcal U^\text{\emph{pre}}_i$,\\
$
\Pr(\Tilde{X}_t.u_\text{\emph{pre}} {=} u_\text{\emph{pre}} | \Tilde{X}_t.s {=} s)
= \pi^\text{\emph{pre}}_i (s, u_\text{\emph{pre}}).
$
\end{prop}
In Property \ref{Phil_property}, we show that  the distribution of $\Tilde{X}_{t+1}.s$ given $\Tilde{X}_t.s$, $\Tilde{X}_t.u_\text{\emph{pre}}$, and $\Tilde{U}_t$ captures the distribution of the $\mathcal{S}$ component of the input that the $\ith$ coagent will see given the input at the previous step and the output that it selected.
\begin{prop}
For all $s \in \mathcal{S}$, $s' \in \mathcal{S}$, $u_\text{\emph{pre}} \in \mathcal U^\text{\emph{pre}}_i$, and $u \in \mathcal U$,
$
\Pr(\Tilde{X}_{t+1}.s {=} s' | \Tilde{X}_t.s {=} s, \Tilde{X}_t.u_\text{\emph{pre}} {=} u_\text{\emph{pre}}, \Tilde{U}_t {=} u)\\
= \Pr(S_{t+1} {=} s' | S_t {=} s, U^\text{\emph{pre}}_t {=} u_\text{\emph{pre}}, U_t {=} u).
$
\end{prop}
In Property \ref{u'_pre_independence}, we 
show that: Given the $\mathcal{S}$ component of the input, the $\mathcal U^\pre_i$ component of the input that the $\ith$ coagent will see is independent of the previous input and output.
\begin{prop}
For all $s \in \mathcal{S}$, $s' \in \mathcal{S}$, $u_\text{\emph{pre}} \in \mathcal U^\text{\emph{pre}}_i$, $u'_\text{\emph{pre}} \in \mathcal U^\text{\emph{pre}}_i$, and $u \in \mathcal U$,
$
    \Pr(\tilde X_{t+1}.u_\text{\emph{pre}} {=} u'_\text{\emph{pre}} | \tilde X_{t+1}.s {=} s')\\
    = \Pr(\tilde X_{t+1}.u_\text{\emph{pre}} {=} u'_\text{\emph{pre}} | \tilde X_{t+1}.s {=} s', \tilde X_{t} {=} (s, u_\text{\emph{pre}}), \tilde U_t {=} u).
$
\end{prop}

In Property \ref{correct_R_t}, we use Properties \ref{correct_s}, \ref{correct_pi_pre}, \ref{Phil_property}, \ref{u'_pre_independence}, and \ref{correct_R_t} to show that  the distribution of $\Tilde{R}^i_t$ captures the distribution of the rewards that the $\ith$ coagent receives.

\begin{prop}
For all $r \in \mathcal R$, 
    $\Pr(R_t {=} r) = \Pr(\Tilde{R}^i_t {=} r).$
\end{prop}

We then use Properties \ref{correct_transition} and \ref{correct_reward} and the definition of $M^i$ to show that:
\begin{lemma}
    $M^i$ is a Markov decision process.
\end{lemma}

Finally, in Lemma \ref{correct_comdp}, we use the properties above to show that the CoMDP $M^i$ (built from $M$, $i$, $\pi^\pre_i$, $\pi^\post_i$, and $\bar \theta_i$) correctly models the local environment of the $\ith$ coagent.
\begin{lemma}
    For all $M,i,\pi_i^\text{\emph{pre}}, \pi_i^\text{\emph{post}}$, and $\bar \theta_i$, and given a policy parameterized by $\theta_i$, the corresponding CoMDP $M^i$ satisfies Properties 1-6 and Property \ref{correct_R_t}.
\end{lemma}

Lemma \ref{correct_comdp} is stated more specifically and formally in the supplementary material.
\subsection{The Coagent Policy Gradient Theorem}
Again, all proofs of the properties and theorems below are provided in Section \ref{sec:complete_CPGT_proofs} of the supplementary material. 
We use Property \ref{correct_R_t} to show that $M$'s objective, $J(\theta)$, is equivalent to the objective $J_i(\theta_i)$ of the $\ith$ CoMDP.
\setcounter{prop}{10}
\begin{prop}
\label{J_equivalence_main_paper}
    For all coagents $i$, for all $\theta_i$, given the same $\theta = (\theta_i, \bar \theta_i)$, $J(\theta) = J_i(\theta_i)$.
\end{prop}

Next, using Lemmas \ref{is_mdp} and \ref{correct_comdp}, we show that the local policy gradient, $\Delta_i$ (the expected value of the naive REINFORCE update), is equivalent to the gradient $\frac{\partial J_i}{\partial \theta_i}$ of the $\ith$ CoMDP.
\setcounter{lemma}{2}
\begin{lemma}
    For all coagents $i$, for all $\theta_i$, $\frac{\partial J_i(\theta_i)}{\partial \theta_i} = \Delta_i(\theta_i).$
\end{lemma}
The CPGT states that the local policy gradients are the components of the global policy gradient.
Notice that this intuitively follows by transitivity from Property \ref{J_equivalence} and Lemma \ref{delta_is_dJ}.
\begin{thm}[Coagent Policy Gradient Theorem]
\label{cpgt}
    $\\ \nabla J(\theta)=\left [\Delta_1(\theta_1)^\intercal , \Delta_2(\theta_2)^\intercal,\dotsc,\Delta_m(\theta_m)^\intercal\right ]^\intercal$, 
    where $m$ is the number of coagents and $\Delta_i$ is the local policy gradient of the $\ith$ coagent.
\end{thm}
\begin{cor}
    If $\alpha_t$ is a deterministic positive stepsize, $\sum_{t=0}^\infty \alpha_t = \infty$, $\sum_{t=0}^\infty \alpha_t^2 < \infty$, additional technical assumptions are met \citep[Proposition 3]{Bertsekas2000}, and each coagent updates its parameters, $\theta_i$, with an unbiased local policy gradient update $\theta_i \gets \theta_i + \alpha_t \widehat \Delta_i(\theta_i)$, then $J(\theta)$ converges to a finite value and $\lim_{t \to \infty} \!\!\nabla J(\theta) {=} 0$. 
\end{cor}
\section{Asynchronous Recurrent Networks}
\label{sec:asynchronous}
Having formally established the CPGT, we now turn to extending the CPGA framework to asynchronous and cyclic networks---networks where the coagents \emph{execute}, that is, look at their local state and choose actions, asynchronously and without any necessary order. 
%
%
This extension allows for distributed implementations, where nodes may not execute synchronously. This also facilitates temporal abstraction, since, by varying coagent execution rates, one can design algorithms that learn and act across different levels of abstraction.

We first consider how we may modify an MDP to allow coagents to execute at arbitrary points in time, including at points \emph{in between} our usual time steps.
Our motivation is to consider continuous time. As a theoretical construct, we can approximate continuous time with arbitrarily high precision by breaking a time step of the MDP into an arbitrarily large number of shorter steps, which we call \emph{atomic time steps}. 
We assume that the environment performs its usual update regularly every $n \in \mathbb Z^+$ atomic time steps, and that each coagent \emph{executes} (chooses an output in its respective $\mathcal U^i$) at each atomic time step with some probability, given by an arbitrary but fixed distribution that may be conditioned on the local state. 
On atomic time steps where the $\ith$ coagent does not execute, it continues to output the last chosen action in $\mathcal U^i$ until its next execution.
%
%
The duration of atomic time steps can be arbitrarily small to allow for arbitrarily close approximations to continuous time or to model, for example, a CPU cluster that performs billions of updates per second.
The objective is still the expected value of $G_0$, the discounted sum of rewards from all atomic time steps: $J(\theta) = \mathbf E[G_0 | \theta] =  \mathbf E[\sum_{t = 0}^\infty \gamma^t R_t | \theta]$.
(Note that $\gamma$ in this equation is the $n^\text{th}$ root of the original $\gamma$, where $n$ is the number of of atomic time steps per environment update.)

Next, we extend the coagent framework to allow \emph{cyclic} connections.
Previously, we considered a coagent's local state to be captured by $X^i_t = (S_t, U^\pre_t)$, where $U^\pre_t$ is some combination of outputs from coagents that come before the $\ith$ coagent topologically.
We now allow coagents to also consider the output of all $m$ coagents 
on the \emph{previous} time step, $U^\all_{t-1} = (U^1_{t-1}, U^2_{t-1}, \dots U^m_{t-1})$.
\,In the new setting, the local state at time $t$ is therefore given by $X^i_t = (S_t, U^\pre_t, U^\all_{t - 1})$.
The  corresponding local state set is given by $\mathcal X^i = \mathcal S \times \mathcal U^\pre \times \mathcal U^1 \times \dots \times \mathcal U^m$.
In this construction, when $t = 0$, we must consider some initial output of each coagent, $U^\all_{-1}$.
For the $\ith$ coagent, we define $U^i_{-1}$ to be drawn from some initial distribution, $h^i_0$, such that for all $u \in \mathcal U^i, h^i_0(u) = \Pr(U^i_{-1} = u)$.

We redefine how each coagent selects actions in the asynchronous setting.
First, we define a random variable, $E^i_t$, the value of which is $1$ if the $\ith$ coagent executes on atomic time step $t$, and $0$ otherwise.
Each coagent has a fixed \emph{execution function}, $\beta_i: \mathcal X^i \times \mathbb N \to [0, 1]$, which defines the probability of the $\ith$ coagent executing on time step $t$, given the coagent's local state. 
That is, for all $x \in \mathcal X^i, \,\, \beta_i(x) \coloneqq \Pr(E^i_t = 1 | X^i_t = x$).
Finally, the action that the $\ith$ coagent selects at time $t$, $U^i_t$, is sampled from $\pi_i(X^i_t, \cdot, \theta_i)$ if $E^i_t = 1$, and is $U^i_{t-1}$ otherwise.
That is, if the agent does not execute on atomic time step $t$, then it should repeat its action from time $t-1$.

We cannot directly apply the CPGT to this setting: the policy and environment are \emph{non-Markovian}.
That is, we cannot determine the distribution over the output of the network given only the current state, $S_t$, since the output may also depend on $U^\all_{t-1}$.
However, we show that the asynchronous setting can be reduced to the acyclic, synchronous setting using formulaic changes to the state set, transition function, and network structure.
This allows us to derive an expression for the gradient with respect to the parameters of the original, asynchronous network, and thus to train such a network.
We prove a result similar to the CPGT that allows us to update the parameters of each coagent using only states and actions from atomic time steps when the coagent executes.
\subsection{The CPGT for Asynchronous Networks}
We first extend the definition of the \emph{local policy gradient}, 
$\Delta_i$, to the asynchronous setting. 
In the synchronous setting, the local policy gradient captures the update that a coagent would perform if it was following an unbiased policy gradient algorithm using its local inputs and outputs.
In the asynchronous setting, we capture the update that an agent would perform if it were to consider only the local inputs and outputs it sees when it executes.
%
%
Formally, we define the \emph{asynchronous local policy gradient}:
%
$
    \Delta_i(\theta_i) \coloneqq
    \mathbf{E} \Big[ \sum_{t=0}^\infty E^i_t \gamma^t G_t 
    %
    \frac{\partial \ln \big( \pi_i \big ( X_t, U_t, \theta_i \big ) \big )}{\partial \theta_i} \Big | \theta \Big].
$

The only change from the synchronous version is the introduction of $E^i_t$.
Note that when the coagent does not execute ($E^i_t = 0$), the entire inner expression is 0.
In other words, these states and actions can be ignored.
An algorithm estimating $G_t$ would still need to consider the rewards from \emph{every} atomic time step, including time steps where the coagent does not execute.
However, the algorithm may still be designed such that the coagents only perform a computation when executing.
For example, during execution, coagents may be given the discounted sum of rewards since their last execution to serve as a summary of all rewards since that execution.
%
%
The important question is then: does something like the CPGT hold for the \emph{asynchronous} local policy gradient? 
If each coagent executes a policy gradient algorithm using unbiased estimates of $\Delta_i$, does the network still perform gradient descent on the asynchronous setting objective, $J$?
The answer turns out to be yes.

\begin{thm}[Asynchronous Coagent Policy Gradient Theorem]
\label{thm:acpgt}
    $\\ \nabla J(\theta)=\left [\Delta_1(\theta_1)^\intercal , \Delta_2(\theta_2)^\intercal,\dotsc,\Delta_m(\theta_m)^\intercal\right ]^\intercal$, 
    where $m$ is the number of coagents and $\Delta_i$ is the asynchronous local policy gradient of the $\ith$ coagent.
\end{thm}
\begin{proof}
The general approach is to show that for any MDP $M$, with an asynchronous network represented by $\pi$ with parameters $\theta$, there is an \emph{augmented} MDP, $\grave M$, with objective $\grave J$ and an \emph{acylic, synchronous} network, $\grave \pi$, with the same parameters $\theta$, such that $J(\theta) = \grave J(\theta)$.
Thus, we \emph{reduce} the asynchronous problem to an equivalent synchronous problem.
Applying the CPGT to this \emph{reduced setting} allows us to derive Theorem \ref{thm:acpgt}.

The original MDP, $M$, is given by the tuple $(\mathcal S, \mathcal A, P, R, d_0, \gamma)$.
We define the \emph{augmented} MDP, $\grave M$, as the tuple, $(\grave{ \mathcal S}, \grave{\mathcal A}, \grave P, \grave R, \grave d_0, \grave \gamma)$.
We would like $\grave M$ to hold \emph{all} of the information necessary for each coagent to compute its next output, including the previous outputs 
of all coagents.
This will allow us to construct an acyclic version of the network to which we may apply the CPGT.
We define $\mathcal U^\text{all} = \mathcal U^1 \times \mathcal U^2 \times \dots \times \mathcal U^m$ to be the combined output set of all $m$ coagents in $\pi$ 
and $\mathcal E = \{0, 1\}^m$ to be the set of possible combinations of coagent executions.
We define the state set to be $\grave{\mathcal S} = \mathcal S \times \mathcal U^\text{all}$ 
and the action set to be $\grave{\mathcal A} = \mathcal A \times \mathcal U^\text{all} \times \mathcal E$.
%
%
We write the random variables representing the state, action, and reward at time $t$ as $\grave S_t, \grave A_t,$ and $\grave R_t$ respectively.
Additionally, we refer to the components of values $\grave s \in \mathcal{\grave S}$ and $\grave a \in \mathcal{\grave A}$ and the components of the random variables $\grave S_t$ and $\grave A_t$ using the same notation as for the components of $X_t$ above (for example, $\grave s.s$ is the $\mathcal S$ component of $\grave s$, $\grave A_t.u^\all$ is the $\mathcal U^\all$ component of $\grave A_t$, etc.).
For vector components, we write the $\ith$ component of the vector using a subscript $i$ (for example, $\grave s.u^\all_i$ is the $\ith$ component of $\grave s.u^\all$).

The transition function, $\grave P$, captures the original transition function and the fact that $\grave S_{t+1}.u^\all = \grave A_t.u^\all$. 
For all $\grave s, \grave s' \in \grave S$ and $\grave a \in \grave A$, $\grave P(\grave s, \grave a, \grave s')$ is given by $P(\grave s.s, \grave a.a, \grave s'.s) \text{ if } \grave s'.u^\all {=} \grave a.u^\all$, 
and $0$ otherwise.
%
%
For all $\grave s, \grave s' \in \grave S, \grave a \in \grave A$, and $r \in \mathbb R$, the reward distribution is simply given by $\grave R(\grave s, \grave a, \grave s', r) {=} R(\grave s.s, \grave a.a, \grave s'.s, r)$. 
The initial state distribution, $\grave d_0$, captures the original state distribution and the initialization of each coagent. 
For all $\grave s {\in} \grave S$, it is given by $\grave d_0(\grave s) {=} d_0(\grave s.s) \prod_{i=1}^m h_0(\grave s.u_i)$. 
The discount parameter is $\grave \gamma {=} \gamma$.
The objective is the usual: $\grave J(\theta) {=} \mathbf E[\grave G_0 | \theta ]$, where $\grave G_0 {=} \mathbf E[\sum_{t=0}^\infty \grave \gamma^t \grave R_t | \theta]$.

Next we define the synchronous network, $\grave \pi$, in terms of components of the original asynchronous network, $\pi$---specifically, each $\pi_i$, $\beta_i$, and $\theta_i$.
We must modify the original network to accept inputs in $\grave{\mathcal S}$ and produce outputs in $\grave{\mathcal A}$.
Recall that in the asynchronous network, the local state at time $t$ of the $\ith$ coagent is given by $X^i_t = (S_t, U^\pre_t, U^\all_{t -1})$.
In the augmented MDP, the information in $U^\all_{t-1}$ is contained in $\grave S_t$, so the local state of the $\ith$ coagent in the synchronous network is $\grave X^i_t = (\grave S_t, \grave U^\pre_t)$, with accompanying state set $\grave{\mathcal X}^i = \grave{\mathcal S} \times \grave{\mathcal U}^\pre$.
To produce the $\mathcal U^\all$ component of the action, $\grave A_t.u^\all$, we append the output of each coagent to the action.
In doing so, we have removed the need for cyclic connections, but still must deal with the asynchronous execution.

The critical step is as follows: We represent each coagent in the asynchronous network by \emph{two} coagents in the synchronous network, the first of which represents the execution function, $\beta_i$, and the second of which represents the original policy, $\pi_i$.
%
At time step $t$, the first coagent accepts $\grave X^i_t$ and outputs $1$ with probability $\beta_i((\grave S_t.s, \grave U^\pre_t, \grave S_t.u))$, and 0 otherwise.
We append the output of every such coagent to the action in order to produce the $\mathcal E$ component of the action, $\grave A_t.e$.
Because the coagent representing $\beta_i$ executes before the coagent representing $\pi_i$, from the latter's perspective, the output of the former is present in $\grave U^\pre_t$, that is, $\grave U^\pre_t.e_i = \grave A_t.e_i$.
If $\grave U^\pre_t.e_i = 1$, the coagent samples a new action from $\pi_i$.
Otherwise, it repeats its previous action, which can be read from its local state (that is, $\grave X^i_t.u^\all_i = \grave U^i_{t-1}$).
Formally, for all $(\grave s, \grave u_\pre) \in \grave{\mathcal X^i}$ and $\theta_i$, the probability of the latter coagent producing action $\grave u \in \grave{\mathcal U^i}$ is given by: $\grave \pi_i((\grave s, \grave u_\pre), \grave u, \theta_i) {\coloneqq}
            \pi_i((\grave s.s, \grave u_\pre, \grave s.u^\all), \grave u, \theta_i) \text{ if } \grave u_\pre.e_i {=} 1,$
            $1 \text{ if } \grave u_\pre.e_i {=} 0 \text{ and } \grave s.u^\all_i {=} \grave u,$
             and $0 \text{ otherwise.}$
%
%
%
This completes the description of $\grave \pi$.
%
In the supplementary material, we prove that this network exactly captures the behavior of the asynchronous network---that is, $\grave \pi((s, u), (a, u', e), \theta)$ $= \Pr(A_t = a, U^\all_{t} = u', E_t = e | S_t = s, U^\all_{t-1} = u, \theta)$ for all possible values of $a, u, a, u', e, \text{ and } \theta$ in their appropriate sets.

The proof that $J(\theta) = \grave J(\theta)$ is given in Section \ref{sec:async_supp} of the supplementary material, but it follows intuitively from the fact that 1) the ``hidden'' state of the network is now captured by the state set, 2) $\grave \pi$ accurately captures the dynamics of the hidden state, and 3) this hidden state does not materially affect the transition function or the reward distribution with respect to the original states and actions.

Having shown that the expected return in the asynchronous setting is equal to the expected return in the synchronous setting, we turn to deriving the asynchronous local policy gradient, $\Delta_i$.
It follows from $J(\theta) = \grave J(\theta)$ that $\nabla J(\theta) = \nabla \grave J(\theta)$.
Since $\grave \pi$ is a synchronous, acylic network, and $\grave M$ is an MDP, we can apply the CPGT to find an expression for  $\nabla \grave J(\theta)$.
For the $\ith$ coagent in the synchronous network, this gives us:
$
    \frac{\partial \grave J(\theta)}{\partial \theta_i} = \mathbf{E} \Big[ \sum_{t=0}^\infty \grave{ \gamma}^t \grave G_t  \frac{\partial \ln\left ( \grave \pi_i\left ( (\grave S_t,\grave U_t^\text{pre}), \grave U_t, \theta_i \right ) \right )}{\partial \theta_i} \Big | \theta \Big].
$

Consider $\partial \ln \big( \grave \pi_i ( (\grave S_t,\grave U_t^\text{pre}), \grave U_t, \theta_i) \big) / \partial \theta_i$, which we abbreviate as $\partial \grave \pi_i / \partial \theta_i$.
When $\grave U^\pre_t.e_i = 0$, we know that the action is $\grave U^i_{t} = \grave S_t.u_i$ $= \grave U^i_{t-1}$ regardless of $\theta$.
Therefore, in these local states, $\partial \grave \pi_i / \partial \theta_i$ is zero.
When $\grave U^\pre_t.e_i = 1$, we see from the definition of $\grave \pi$ that $\partial \grave \pi_i / \partial \theta_i {=} \partial \pi_i / \partial \theta_i$.
Therefore, we see that in all cases, $\partial \grave \pi_i / \partial \theta_i {=} (\grave U^\pre_t.e_i) \partial \pi / \partial \theta_i$.
Substituting this into the above expression yields:
$
    \mathbf{E} \Big[ \!\sum_{t=0}^\infty (\grave U^\pre_t.e_i) \grave \gamma^t  \grave G_t  \frac{\partial \ln\big ( \pi_i \big( (\grave S_t.s,\grave U_t^\text{pre}, \grave S_t.u^\all), \grave U_t, \theta_i \big) \big )}{\partial \theta_i} \Big | \theta \Big]. 
$

In the proof that $J(\theta) = \grave J(\theta)$ given in Section \ref{sec:async_supp} of the supplementary material, we show that the distribution over all analogous random variables is equivalent in both settings (for example, for all $s \in \mathcal S$, $\Pr(S_t=s) = \Pr(\grave S_t.s = s)$).
Substituting each of the random variables of $M$ into the above expression yields precisely the asynchronous local policy gradient, $\Delta_i$.
\end{proof}
To empirically test the Asynchronous Coagent Policy Gradient Theorem (ACPGT), we compare the gradient ($\nabla J$) estimates of the ACPGT with a finite difference method. The results are presented in Figure \ref{finite_dif} (Section \ref{sec:finite_dif_details}) of the supplementary material; this data provides empirical support for the ACPGT.
\section{Applications to Past and Future Work}
Consider past RL approaches that use stochastic networks, such as stochastic computation graphs, hierarchical networks like the option-critic, or any other form of SNN; one could use the ACPGT to immediately obtain a learning rule for any of these varieties of SNN.
In other words, the theory in the sections above facilitates the derivation of gradient and learning rules for \emph{arbitrary} stochastic architectures.
%
%
The ACPGT applies when there are recurrent connections.
It applies in the asynchronous case, even when the units (coagents) have different rates or execution probabilities, and/or have execution functions that depend on the output of other units or the state.
It also applies to architectures that are designed for temporal abstraction, like the option-critic depicted in Figure \ref{fig:option-critic}. Architectures adding additional levels of abstraction, such as the \emph{hierarchical option-critic} \citep{NIPS2018_8243}, can be analyzed with similar ease.
Architectures designed for other purposes, such as partitioning the state space between different parts of the network \citep{Sutton2011}, or simplifying a high-dimensional action space \citep{Thomas2012}, can also be analyzed with the coagent framework.

The diversity of applications is not limited to network topology variations.  For example, one could design an asynchronous deep neural network where each unit is a separate coagent.
Alternatively, one could design an asynchronous deep neural network where each coagent is itself a deep neural network.
In both cases, the coagents could run at different rates to facilitate temporal abstraction.
In the latter case, the gradient generated by an algorithm following the ACPGT could be used to train each coagent internally with backpropagation.
Notice that the former architecture is trained without backpropagation while the latter architecture combines an ACPGT algorithm with backpropagation: The ACPGT facilitates easy design and analysis for both types of architectures and algorithms.

Deriving the policy gradient for a particular coagent simply requires identifying the inputs and outputs and plugging them into the ACPGT formula.
In this way, an algorithm designer can rapidly design a principled policy gradient algorithm for any stochastic network, even in the asynchronous and recurrent setting.
In the next section, we give an example of this process.
\section{Case Study: Option-Critic}
\label{sec:option}
The coagent framework allows one to design an arbitrary hierarchical architecture, and then immediately produce a learning rule for that architecture.
In this section, we study the well-known \emph{option-critic} framework \citep{bacon2017option} as an example, demonstrating how the CPGT drastically simplifies the gradient derivation process.
The \emph{option-critic} framework 
aspires to many of the same goals as coagent networks: namely, hierarchical learning and temporal abstraction.
We show that the architecture is equivalently described in terms of a simple, three-node coagent network, depicted and described in Figure \ref{fig:option-critic}.
More formal and complete definitions are provided in Section \ref{sec:option-critic-complete-description} of the supplemental material, but Figure \ref{fig:option-critic} is sufficient for the understanding of this section.
\begin{figure}
    \centering
    \includegraphics[width=.49\linewidth]{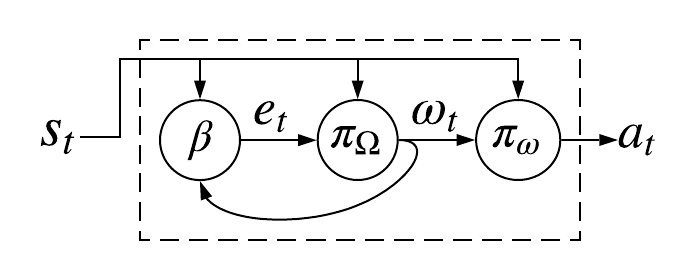}
    \caption{The option-critic framework, described by \citet{bacon2017option}, depicted as a three-coagent network consisting of \textbf{1)} $\beta$, which contains the termination functions for each option, \textbf{2)} $\pi_\Omega$, the policy over options, and \textbf{3)} $\pi_\omega$, which contains the option policies and selects the action $a_t$.
    Each option has a corresponding \emph{label}.
    $\Omega$ is the set of the available options' labels.
    The label corresponding to the option selected at time $t$ is denoted as $\omega_t$. 
    $\beta$'s local state is $s_t$ and the previous option label, $\omega_{t-1}$ (the latter via a recurrent connection from $\pi_\Omega$). $\beta$ sends an action, $e_t \in \{0,1\}$, to $\pi_\Omega$. $\pi_\Omega$'s execution function outputs $1$ (a $100\%$ chance of execution) if $e_t = 1$, and outputs $0$ (no chance of execution) if $e_t = 0$.
    That is, $\pi_\Omega$ executes if and only if $e_t = 1$.
    $\pi_\Omega$'s local state is $s_t$ and $e_t$. If it executes, $\pi_\Omega$ outputs a new option label $\omega_t \in \Omega$.
    $\pi_\omega$ chooses an action, $a_t \in \mathcal A$, based on its local state, $s_t$ and $\omega_t$.
    Internally, $\pi_\omega$ looks up the option policy corresponding to the option label $\omega_t$ and selects an action based on that option.
    }
    \label{fig:option-critic}
\end{figure}

\citet{bacon2017option} give gradients for two of the three coagents. $\pi_\omega$, the policy that selects the action, and $\beta$, the termination functions, are parameterized by weights $\theta$ and $\vartheta$, respectively.
\citet{bacon2017option} gave the corresponding policy gradients, which we rewrite as:
\begin{align}
    \frac{\partial J}{\partial \theta} =& \sum_{x \in (\mathcal S \times \Omega)} d^\pi_\Omega(x) \sum_{a \in \mathcal A} \frac{\partial \pi_\omega(x, a)}{\partial \theta} Q_U(x, a), \label{eq1}
\end{align}
\begin{align}
    \frac{\partial J}{\partial \vartheta} =& - \sum_{x \in (\mathcal S \times \Omega)} d^\pi_\Omega(x) \frac{\partial \beta(x, 0)}{\partial \vartheta} A_\Omega(x.s, x.\omega), \label{eq2}
\end{align}
where $d^\pi_\Omega(x)$ is a discounted weighting of state-option pairs, given by $d^\pi_\Omega(x) \coloneqq \sum_{t = 0}^\infty \gamma^t \Pr(s_t {=} x.s, \omega_t {=} x.\omega)$, $Q_U(x, a)$ is the expected return from choosing option $x.\omega$ and action $a$ at state $s$ under the current policy, and $A_\Omega(s, \omega)$ is the advantage of choosing option $\omega$, given by $A_\Omega(s, \omega) = Q_\Omega(s, \omega) - V_\Omega(s)$, where $Q_\Omega(s, \omega)$ is the expected return from choosing option $\omega$ in state $s$, and $V_\Omega(s)$ is the expected return from beginning in state $s$ with no option selected.

Previously, the CPGT was written in terms of expected values.
An equivalent expression of the local gradient for policy $\pi_i$ is the sum over the local state set, $\mathcal X_i$, and the local action set, $\mathcal U_i$:
$
    \partial J(\theta)/
    \partial \theta_i
    = \sum_{x \in \mathcal X_i} d^\pi_i(x) \sum_{u \in \mathcal U_i} \frac{\partial \pi_i(x, u)}{\partial \theta_i} Q_i(x, u),
$
where $Q_i(x, u) = \mathbf E[G_t | X^i_t {=} x, U^i_t {=} u]$.
Deriving the policy gradient for a particular coagent simply requires identifying the inputs and outputs and plugging them into this formula.
%
%

First consider the policy gradient for $\pi_\omega$, that is, $\partial J/\partial \theta$: the input set is $X_\omega = \mathcal S \times \Omega$, and the action set is $\mathcal A$.
The local initial state distribution (the $d^\pi_i$ term) is given exactly by $d^\pi_\Omega$, and the local state-action value function (the $Q_i$ term) is given exactly by $Q_U$.
Directly substituting these terms into the CPGT immediately yields 
\eqref{eq1}.
Note that this derivation is completely trivial using the CPGT: only direct substitution is required. 
In contrast, the original derivation from \citet{bacon2017option} required a degree of complexity and several steps.

Next consider $\partial J/\partial \vartheta$.
The input set is again $X_\beta = \mathcal S \times \Omega$, but the action set is $\{0, 1\}$.
The local state distribution is again the distribution over state-option pairs, given by $d^\pi_\Omega$.
The PCGN expression gives us 
\begin{align}
    \frac{\partial J}{\partial \vartheta} = \sum_{x \in (\mathcal S \times \Omega)} d^\pi_\Omega(x) \sum_{u \in \{0, 1\}} \frac{\partial \beta(x, u)}{\partial \vartheta} Q_\beta(x, u).
\end{align}
In Section \ref{sec:option-critic-equiv} of the supplementary material we show that this is equivalent to \eqref{eq2}.
Notice that, unlike $\pi_\Omega$, both of these coagents execute every atomic time step. The execution function, therefore, always produces $1$ and therefore can be ignored for the purposes of the gradient calculation.

Finally, consider the gradient of $J$ with respect to the parameters of $\pi_\Omega$.
\citet{bacon2017option} do not provide a policy gradient for $\pi_\Omega$, but suggest policy gradient methods at the SMDP level, planning, or \emph{intra-option Q-learning} \citep{sutton1999between}.
One option is to use the ACPGT approach.
The execution function is simply the output of the termination coagent at time $t$, $e_t$.
We use the expected value form for simplicity, substituting terms in as above:
\begin{align}
\partial J / \partial \mu = \mathbf{E}\Big[\sum_{t=0}^\infty e_t \gamma^t G_t \frac{\partial \ln \big( \pi_\Omega \big ( X_t, \omega_t, \mu \big ) \big )}{\partial \mu} \Big | \{\theta, \vartheta, \mu\} \Big],
\end{align}
where $\mu$ represents the parameters of $\pi_\Omega$ and $X_t \in (\mathcal S \times \{0,1\})$ is the local state of the $\pi_\Omega$ coagent at time $t$.
While \citet{bacon2017option} introduce an asynchronous framework, the theoretical tools they use for the synchronous components do not provide a gradient for the asynchronous component.
Instead, their suggestions rely on a piecemeal approach.  While this approach is reasonable, it invokes ideas beyond the scope of their work for training the asynchronous component.
Our approach has the benefit of providing a unified approach to training all components of the network.

The ACPGT provided a simplified and unified approach to these gradient derivations. 
Using the option-critic framework as an example, we have shown that the ACPGT is a useful tool for analyzing arbitrary stochastic networks. Notice that a more complex architecture containing many levels of hierarchy could be analyzed with similar ease.

\section{Conclusion}
We provide a formal and general proof of the coagent policy gradient theorem (CPGT) for stochastic policy networks, and extend it to the asynchronous and recurrent setting.
This result demonstrates that, if coagents apply standard policy gradient algorithms from the perspective of their inputs and outputs, then the entire network will follow the policy gradient, even in asynchronous or recurrent settings.
We empirically support the CPGT, and use the option-critic framework as an example to show how our approach facilitates and simplifies gradient derivation for arbitrary stochastic networks.
Future work will focus on the potential for massive parallelization of asynchronous coagent networks, and on the potential for many levels of implicit temporal abstraction through varying coagent execution rates.

\section*{Acknowledgements}

Research reported in this paper was sponsored in part by the CCDC Army Research Laboratory under Cooperative Agreement W911NF-17-2-0196 (ARL IoBT CRA). The views and conclusions contained in this document are those of the authors and should not be interpreted as representing the official policies, either expressed or implied, of the Army Research Laboratory or the U.S. Government. The U.S. Government is authorized to reproduce and distribute reprints for Government purposes notwithstanding any copyright notation herein.

We would like to thank Scott Jordan for help and feedback with the editing process.
We would also like to thank the reviewers and meta-reviewers for their feedback, which helped us improve this work.

\bibliography{example_paper}
\bibliographystyle{icml2020}

\appendix
\newpage
\onecolumn

\setcounter{prop}{0} 
\setcounter{lemma}{0} 
\setcounter{thm}{0} 

\section{Conjugate Markov Decision Process (CoMDP)}
\label{sec:complete_comdp}
In order to reason about the local policy gradient, we begin by modeling the $\ith$ coagent's environment as an MDP, called the CoMDP, and begin by formally defining the $\ith$ CoMDP. 
Given $M$, $i$, $\pi^\pre_i$, $\pi^\post_i$, and $\bar \theta_i$, we define a corresponding CoMDP, $M^i$, as $M^i\coloneqq(\mathcal X^i,\mathcal U^i, \mathcal R^i, P^i, R^i, d_0^i, \gamma_i)$, where:
\begin{itemize}
    \item We write $\Tilde{X}_t^i$, $\Tilde{U}_t^i$, and $\Tilde{R}_t^i$ to denote the state, action, and reward of $M^i$ at time $t$. Below, we relate these random variables to the corresponding random variables in $M$. Note that all \emph{random variables} in the CoMDP are written with tildes to provide a visual distinction between terms from the CoMDP and original MDP.  Additionally, when it is clear that we are referring to the $\ith$ CoMDP, we often make $i$ implicit and denote these as $\Tilde{X}_t$, $\Tilde{U}_t$, and $\Tilde{R}_t$.
    \item $\mathcal X^i \coloneqq \mathcal S \times \mathcal U^\pre_i$.  We often denote $\mathcal X^i$ simply as $\mathcal X$.  This is the input (analogous to a state set) to the $\ith$ coagent.  Additionally, for $x \in \mathcal{X}$, we denote the $\mathcal{S}$ component as $x.s$ and the $\mathcal U^\pre$ component as $x.u_\pre$.  We also sometimes denote an $x \in \mathcal{X}^i$ as $(x.s, x.u_\pre)$.  For example, $\Pr(\Tilde{X}^i_t {=} (s, u_\pre))$ represents the probability that $\Tilde{X}_t$ has $\mathcal{S}$ component $s$ and $\mathcal U^\pre$ component $u_\pre$.
    \item $\mathcal U^i$ (or simply $\mathcal{U}$) is an arbitrary set that denotes the output of the $\ith$ coagent.
    \item $\mathcal R^i\coloneqq\mathcal R$ and $\gamma_i\coloneqq \gamma$. 
    \item $\forall x\in \mathcal X \,\, \forall x' \in \mathcal X \,\, \forall u \in \mathcal U \,\, \forall \bar \theta_i \in \mathbb R^{n-n_i},$
    \begin{align*}
      P^i(x,u,x',\bar\theta_i)\coloneqq \pi^\pre_i(x'\!.s,x'.u_\pre) \sum_{a \in \mathcal A} P(x.s, a,x'.s) \pi^\post_i(x,u,a),
    \end{align*}
    %
    Below, we make $\bar\theta_i$ implicit and denote this as $P^i(x,u,x')$.
    Recall from the definition of an MDP and its relation to the transition function that this means: $P^i(x,u,x')=\Pr(\Tilde{X}_{t+1}{=}x'|\Tilde{X}_t{=}x,\Tilde{U}_t{=}u)$.
    \item
    $\forall x\in \mathcal X \,\, \forall x' \in \mathcal X \,\, \forall u \in \mathcal U \,\, \forall r \in \mathcal R^i \,\, \forall \bar \theta_i \in \mathbb R^{n-n_i},$
    \begin{align*}
       R^i(x,u,x',r,\bar\theta_i)\coloneqq \sum_{a \in \mathcal A} R(x.s,a,x'.s,r)
        %
        \frac{P(x.s,a,x'.s)\pi_i^\post(x,u,a)}{\sum_{\hat a \in \mathcal A} P(x.s,\hat a,x'.s)\pi_i^\post(x,u,\hat a)}.
    \end{align*}
    Like the transition function, we make $\bar\theta_i$ implicit and write $R^i(x,u,x',r)$.
    \item $\forall x\in \mathcal X, \,\, d_0^i(x) \coloneqq d_0(x.s) \pi^\pre_i(x.s,x.u_\pre)$.
\end{itemize}

We write $J_i(\theta_i)$ to denote the objective function of $M^i$. 
Notice that although $\bar{\theta}_i$ (the parameters of the other coagents) is not an explicit parameter of the objective function, it is implicitly included via the CoMDP's transition function.
Note that we cannot assume that, for all $\theta_i$, $\Delta_i(\theta_i)$ (the local policy gradient) is equivalent to $\partial J_i(\theta_i) / \partial \theta_i$ (the policy gradient of the $\ith$ CoMDP); we do later prove this equivalence.

\section{Complete CPGT Proofs}
\label{sec:complete_CPGT_proofs}

\begin{figure}
    \centering
    \includegraphics[width=.7\textwidth]{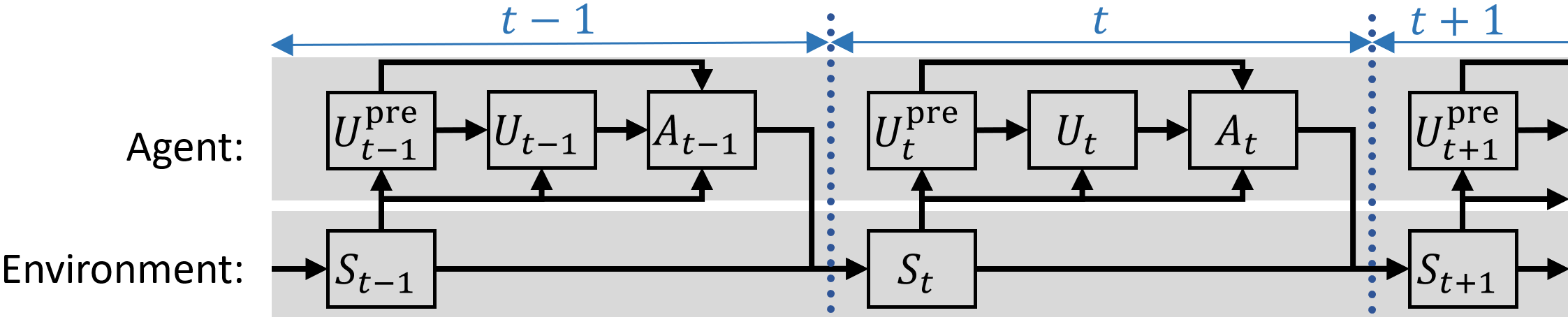}
    \caption{Diagram of the MDP and parameterized policy. This diagram can be viewed as a Bayesian network where each square is the node corresponding to the listed random variable.
    Bayesian network depicting the relationships of relevant random variables. 
    Independence properties can be established by $d$-separation. 
    Note that these causal properties only apply to the MDP $M$; any such properties of CoMDPs are explicitly proven.}
    \label{fig:Causal_supp}
\end{figure}

We assume that, given the same parameters $\theta_i$, the $\ith$ coagent has the same policy in both the original MDP and the $\ith$ CoMDP. That is,

\begin{assumption}
\label{supp_agent_equivalence}
$\forall s \in \mathcal S \,\, \forall u_\text{\emph{pre}} \in \mathcal U^\text{\emph{pre}} \,\, \forall u \in \mathcal U \,\, \forall \theta_i \in \mathbb R^i, \,\, \pi_i ((s, u_\text{\emph{pre}}), u, \theta_i) = \Pr(\Tilde{U}_t = u | \Tilde{X_t} = (s, u_\text{\emph{pre}}), \theta_i)$.
\end{assumption}

\begin{prop}
\label{correct_di0}
$$
\forall x \in \mathcal X, \, d_0^i(x)=\Pr(S_0=x.s, U_0^\text{\emph{pre}}=x.u_\text{\emph{pre}}).
$$
\end{prop}
\begin{proof}
\begin{align}
    d_0^i(x) =& d_0(x.s) \pi^\pre_i(x.s,x.u_\pre)\\
    \overset{\text{(a)}}{=}& \Pr(S_0 = x.s) \Pr(U^\pre_0 = x.u_\pre | S_0 = x.s)\\
    =& \Pr(S_0 = x.s, U^\pre_0 = x.u_\pre),   
\end{align}
where \textbf{(a)} follows from the definitions of $\pi_i^\text{pre}$ and $d_0$. 
\end{proof}

\begin{prop}
$$
\forall s \in \mathcal S, \,\, \Pr(\Tilde X_0.s=s)=d_0(s).
$$
\end{prop}
\begin{proof}
\begin{align}
    \Pr(\Tilde X_0.s=s) \overset{\text{(a)}}{=}& \sum_{u_\pre \in \mathcal U^\pre} \Pr(\Tilde X_0.s=s,\Tilde X_0.u_\pre=u_\pre)\\
    \overset{\text{(b)}}{=}& \sum_{u_\pre \in \mathcal U^\pre} d_0^i((s,u_\pre))\\
    \overset{\text{(c)}}{=}&\sum_{u_\pre \in \mathcal U^\pre} d_0(s) \pi^\pre_i(s,u_\pre)\\
    =&d_0(s)\underbrace{\sum_{u_\pre \in \mathcal U^\pre} \Pr(U_t^{\pre}=u_\pre|S_t=s)}_{=1}\\
    =&d_0(s),
\end{align}
where \textbf{(a)} follows from marginalization over $u_\pre$, \textbf{(b)} follows from the definition of the initial state distribution for an MDP, and \textbf{(c)} follows from the definition of $d_0^i$ for the CoMDP (see Property \ref{correct_di0}). 
\end{proof}

\begin{prop}
\label{correct_transition}
$$
\forall x\in \mathcal X \,\, \forall x' \in \mathcal X \,\, \forall u \in \mathcal U,  \,\, P^i(x,u,x')=\Pr(S_{t+1}=x'.s,U^\text{\emph{pre}}_{t+1}=x'.u_\text{\emph{pre}}|S_{t}=x.s,U^\text{\emph{pre}}_{t}=x.u_\text{\emph{pre}}, U_t = u).
$$
\end{prop}
\begin{proof}
\begin{align}
    P^i(x,u,x')=&\pi^\pre_i(x'\!.s,x'.u_\pre)\sum_{a \in \mathcal A} P(x.s, a,x'.s) \pi^\post_i(x,u,a)\\
    \overset{\text{(a)}}{=}& \sum_{a \in \mathcal A} \pi^\post_i(x,u,a) \Pr(S_{t+1} = x'.s | S_t = x.s, A_t = a) \Pr(U^\pre_{t+1} = x'.u_\pre | S_{t+1} = x'.s)\\
    \overset{\text{(b)}}{=}& \sum_{a \in \mathcal A} \pi^\post_i(x,u,a) \Pr(S_{t+1} = x'.s | S_t = x.s, A_t = a)\\
    &\times \Pr(U^\pre_{t+1} = x'.u_\pre | S_{t+1} = x'.s, S_t = x.s, A_t = a),\\
    \end{align}
    where \textbf{(a)} follows from the definitions of $\pi^\pre_i$ and the transition function $P$, \textbf{(b)} follows from $M$'s conditional independence properties, and $\times$ denotes scalar multiplication split across two lines.  The definition of conditional probability allows us to combine the last two terms:
    
    \begin{align}
    P^i(x,u,x') =& \sum_{a \in \mathcal A} \pi^\post_i(x,u,a) \Pr(S_{t+1} = x'.s, U^\pre_{t+1} = x'.u_\pre | S_t = x.s, A_t = a) \\
    \overset{\text{(a)}}=& \sum_{a \in \mathcal A} \Pr(A_t = a | S_t = x.s, U^\pre_{t} = x.u_\pre, U_t=u)\\
    &\times \Pr(S_{t+1} = x'.s, U^\pre_{t+1} = x'.u_\pre | S_t = x.s, U^\pre_{t} = x.u_\pre, U_t=u, A_t = a)\\
    \overset{\text{(b)}}=& \Pr(S_{t+1} = x'.s, U^\pre_{t+1} = x'.u_\pre |  S_t = x.s, U^\pre_{t} = x.u_\pre, U_t=u),
\end{align}
where \textbf{(a)} follows from the definition of $\pi^\post_i$ and the application of $M$'s independence properties and \textbf{(b)} follows from marginalization over $a$.
\end{proof}

\begin{prop}
\label{correct_reward}
\begin{align}
&\forall x\in \mathcal X \,\, \forall x' \in \mathcal X \,\, \forall u \in \mathcal U  \,\,  \forall r \in \mathcal{R},  \,\, R^i(x,u,x',r)\\
=&\Pr(R_t = r|S_{t}=x.s,U^\text{\emph{pre}}_{t}=x.u_\text{\emph{pre}}, U_t = u, S_{t+1}=x'.s, U^\text{\emph{pre}}_{t+1}=x'.u_\text{\emph{pre}}).
\end{align}
\end{prop}
\begin{proof}
\begin{align}
    R^i(x,u,x',r) \coloneqq& \sum_{a \in \mathcal A} R(x.s,a,x'.s,r) \frac{P(x.s,a,x'.s)\pi_i^\post(x,u,a)}{\sum_{\hat a \in \mathcal A} P(x.s,\hat a,x'.s)\pi_i^\post(x,u,\hat a)}\\
    \overset{\text{(a)}}{=}&\sum_{a \in \mathcal A} R(x.s,a,x'.s,r) P(x.s,a,x'.s)\pi_i^\post(x,u,a)\\
    &\div  \bigg[\sum_{\hat a \in \mathcal A} \Pr(S_{t+1}=x'.s |S_t=x.s, A_t =\hat a, U^\pre_t = x.u_\pre,U_t = u)\\
    &\times \Pr(A_t = \hat a| S_t=x.s, U^\pre_t = x.u_\pre,U_t = u)\bigg]\\
    \overset{\text{(b)}}{=}&\frac{\sum_{a \in \mathcal A} R(x.s,a,x'.s,r) P(x.s,a,x'.s)\pi_i^\post(x,u,a)}{\Pr(S_{t+1}=x'.s |S_t=x.s, U^\pre_t = x.u_\pre,U_t = u)},
\end{align}
    where \textbf{(a)} follows from the definitions of terms in the denominator and $M$'s conditional independence properties (applied to the first term in the denominator) and \textbf{(b)} follows from marginalization over $\hat a$.  Expanding the definitions of the remaining terms, we get:
\begin{align}
     R^i(x,u,x',r)
     =&\frac{\sum_{a \in \mathcal A} \Pr(R_t=r | S_t = x.s, A_t = a, S_{t+1} = x'.s) \Pr(S_{t+1}=x'.s |S_t=x.s, A_t = a)}{\Pr(S_{t+1}=x'.s |S_t=x.s, U^\pre_t = x.u_\pre,U_t = u)}\\
    &\times \Pr(A_t = a| S_t=x.s, U^\pre_t = x.u_\pre,U_t = u)\\
    \overset{\text{(a)}}{=}&\frac{1} {\Pr(S_{t+1}=x'.s |S_t=x.s, U^\pre_t = x.u_\pre,U_t = u)}\\
    &\times \sum_{a \in \mathcal A} \Pr(R_t=r | S_t = x.s, A_t = a, S_{t+1} = x'.s, U^\pre_t = x.u_\pre,U_t = u)\\
    &\times \Pr(S_{t+1}=x'.s |S_t=x.s, A_t = a, U^\pre_t = x.u_\pre,U_t = u)\\
    &\times \Pr(A_t = a| S_t=x.s, U^\pre_t = x.u_\pre,U_t = u),\\
    \end{align}
    where \textbf{(a)} follows from $M$'s conditional independence properties (applied to the $\Pr(R_t=r | ...)$ and $\Pr(S_{t+1}=x'.s |...)$ terms).  Rearranging and taking advantage of marginalization over $a$ (the $\Pr(R_t=r |S_{t+1}=x'.s, ...)$ and $\Pr(S_{t+1}=x'.s |...)$ terms can be viewed as a union), we get:
    \begin{align}
     R^i(x,u,x',r)
    =&\frac {\Pr(S_{t+1}=x'.s |S_t=x.s, U^\pre_t = x.u_\pre,U_t = u)} {\Pr(S_{t+1}=x'.s |S_t=x.s, U^\pre_t = x.u_\pre,U_t = u)}\\
    &\times \Pr(R_t=r | S_t = x.s, S_{t+1} = x'.s, U^\pre_t = x.u_\pre,U_t = u)\\
    =& \Pr(R_t=r | S_t = x.s, S_{t+1} = x'.s, U^\pre_t = x.u_\pre,U_t = u)\\
    \overset{\text{(a)}}{=}& \Pr(R_t=r | S_{t}=x.s,U^\pre_{t}=x.u_\pre, U_t = u, S_{t+1}=x'.s, U^\pre_{t+1}=x'.u_\pre),
\end{align}
where \textbf{(a)} follows from $M$'s conditional independence properties.
\end{proof}

\begin{prop}
\label{correct_x}
\begin{align}
\forall s \in \mathcal{S} \,\, \forall u_\text{\emph{pre}} \in \mathcal U^\text{\emph{pre}}_i, \, \Pr(\Tilde{X}_t = (s, u_\text{\emph{pre}})) = \Pr(S_t = s, U^\text{\emph{pre}}_t = u_\text{\emph{pre}}).
\end{align}
\end{prop}
\begin{proof}
\ \\We present a proof by induction:\\
Base Case: \\
\begin{align}
    \Pr(S_0 = s, U^\pre_0 = u_\pre) =& \Pr(S_0 = s) \Pr(U^\pre_0 = u_\pre | S_0 = s) \\
    =& d_0(s) \pi^\pre_i(s, u_\pre)\\
    =& d^i_0((s, u_\pre))\\
    =& \Pr(\Tilde{X}^i_0 = (s, u_\pre)).
\end{align}
Inductive Step:\\
\begin{align}
    \Pr(S_{t+1} = s',  U^\pre_{t+1} = u'_\pre)
    \overset{\text{(a)}}{=}&
    \sum_{(s,u_\pre) \in \mathcal{X}} \Pr(S_{t} = s,  U^\pre_{t} = u_\pre) \Pr(S_{t+1} = s',  U^\pre_{t+1} = u'_\pre | S_{t} = s,  U^\pre_{t} = u_\pre) \\
    \overset{\text{(b)}}{=}& \sum_{(s,u_\pre) \in \mathcal{X}}
    \Pr(S_{t} = s,  U^\pre_{t} = u_\pre)
    \sum_{u \in \mathcal{U}}
    \Pr(U_t = u | S_{t} = s,  U^\pre_{t} = u_\pre) \\
    &\times
    \Pr(S_{t+1} = s',  U^\pre_{t+1} = u'_\pre | S_{t} = s,  U^\pre_{t} = u_\pre, U_t = u) \\
    \overset{\text{(c)}}{=}& \sum_{(s,u_\pre) \in \mathcal{X}}
    \Pr(\Tilde{X}_{t} = (s, u_\pre))
    \sum_{u \in \mathcal{U}}
    \Pr(\Tilde{U}_t = u | \Tilde{X}_{t} = (s, u_\pre)) \\
    &\times
    \Pr(S_{t+1} = s',  U^\pre_{t+1} = u'_\pre | S_{t} = s,  U^\pre_{t} = u_\pre, U_t = u),
\end{align}
where \textbf{(a)} follows from marginalization over $(s,u_\pre)$, \textbf{(b)} follows from marginalization over $u$, and \textbf{(c)} is through application of the base case and Assumption \ref{supp_agent_equivalence}.  Notice that the last term is equivalent to $P^i$ by Property \ref{correct_transition}, which is equivalent to the final term in the next step:
\begin{align}
    \Pr(S_{t+1} = s',  U^\pre_{t+1} = u'_\pre) =& \sum_{(s,u_\pre) \in \mathcal{X}}
    \Pr(\Tilde{X}_{t} = (s, u_\pre))
    \sum_{u \in \mathcal{U}}
    \Pr(\Tilde{U}_t = u | \Tilde{X}_{t} = (s, u_\pre)) \\
    &\times
    \Pr(\Tilde{X}_{t+1} = (s', u'_\pre) | \Tilde{X}_{t} = (s, u_\pre), \Tilde{U}_t = u) \\
    \overset{\text{(a)}}{=}& \sum_{(s,u_\pre) \in \mathcal{X}}
    \Pr(\Tilde{X}_{t} = (s, u_\pre))
    \Pr(\Tilde{X}_{t+1} = (s', u'_\pre) | \Tilde{X}_{t} = (s, u_\pre)) \\
    \overset{\text{(b)}}{=}& \Pr(\Tilde{X}_{t+1} = (s', u'_\pre)),
\end{align}
where \textbf{(a)} and \textbf{(b)} follow from marginalization over $u$ and $(s,u_\pre)$, respectively.
\end{proof}

\begin{prop}
\label{correct_s}
\begin{align}
\forall s \in \mathcal{S}, \Pr(S_t = s) = \Pr(\Tilde{X}_t.s = s).
\end{align}
\end{prop}
\begin{proof}
\begin{align}
    \Pr(S_t = s) =& \sum_{u_\pre \in \mathcal U^\pre_i}
    \Pr(S_t = s, U^\pre_t = u_\pre)\\
    \overset{\text{(a)}}{=}& \sum_{u_\pre \in \mathcal U^\pre_i} \Pr(\Tilde{X}_t = (s, u_\pre))\\
    \overset{\text{(b)}}{=}& \Pr(\Tilde{X}_t.s = s),
\end{align}
where \textbf{(a)} follows from Property \ref{correct_x} and \textbf{(b)} follows from marginalization over $u_\pre$.
\end{proof}

\begin{prop}
\label{correct_pi_pre}
$\forall s \in \mathcal{S} \,\, \forall u_\text{\emph{pre}} \in \mathcal U^\text{\emph{pre}}_i, \,
\pi^\text{\emph{pre}}_i (s, u_\text{\emph{pre}}) = \Pr(\Tilde{X}_t.u_\text{\emph{pre}} = u_\text{\emph{pre}} | \Tilde{X}_t.s = s)$.
\end{prop}
Recall that $\pi^\pre_i (s, u_\pre) \coloneqq \Pr(U^\pre_t = u_\pre | S_t = s)$.
\begin{proof}
\begin{align*}
    \pi^\pre_i (s, u_\pre) =& \Pr(U^\pre_t = u_\pre | S_t = s) \\
    =& \frac{\Pr(U^\pre_t = u_\pre, S_t = s)} {\Pr(S_t = s)} \\ 
    \overset{\text{(a)}}{=}& \frac{\Pr(\Tilde{X}_t.u_\pre = u_\pre, \Tilde{X}_t.s = s)} {\Pr(\Tilde{X}_t.s = s)} \\ 
    =& \Pr(\Tilde{X}_t.u_\pre = u_\pre | \Tilde{X}_t.s = s),
\end{align*}
where \textbf{(a)} follows from properties \ref{correct_x} and \ref{correct_s}.
\end{proof}

\begin{prop}
\label{Phil_property}
\begin{align}
    &\forall s \in \mathcal{S} \,\, \forall s' \in \mathcal{S} \,\, \forall u_\text{\emph{pre}} \in \mathcal U^\text{\emph{pre}}_i \,\, \forall u \in \mathcal U, \,\\
&\Pr(\Tilde{X}_{t+1}.s = s' | \Tilde{X}_t.s = s, \Tilde{X}_t.u_\text{\emph{pre}} = u_\text{\emph{pre}}, \Tilde{U}_t = u) = \Pr(S_{t+1} = s' | S_t = s, U^\text{\emph{pre}}_t = u_\text{\emph{pre}}, U_t = u).
\end{align}
\end{prop}
\begin{proof}
\begin{align}
    \Pr(&S_{t+1} = s' | S_t = s, U^\pre_t = u_\pre, U_t = u)\\
    \overset{\text{(a)}}{=}& \sum_{a \in \mathcal A} \pi^\post_i ((s, u_\pre), u, a) \Pr(S_{t+1} = s' | S_t = s, U^\pre_t = u_\pre, U_t = u, A_t = a)\\
    \overset{\text{(b)}}{=}& \sum_{a \in \mathcal A} \pi^\post_i ((s, u_\pre), u, a) \sum_{u'_\pre \in \mathcal U^\pre} \Pr(U^\pre_{t+1} = u'_\pre | S_t = s, U^\pre_t = u_\pre, U_t = u, A_t = a)\\
    &\times \Pr(S_{t+1} = s' | S_t = s, U^\pre_t = u_\pre, U_t = u, A_t = a, U^\pre_{t+1} = u'_\pre)\\
    \overset{\text{(c)}}{=}& \sum_{a \in \mathcal A} \pi^\post_i ((s, u_\pre), u, a) \sum_{u'_\pre \in \mathcal U^\pre} \Pr(U^\pre_{t+1} = u'_\pre | S_t = s, U^\pre_t = u_\pre, U_t = u, A_t = a, S_{t+1} = s')\\
    &\times \Pr(S_{t+1} = s' | S_t = s, U^\pre_t = u_\pre, U_t = u, A_t = a)\\
    \overset{\text{(d)}}{=}& \sum_{a \in \mathcal A} \pi^\post_i ((s, u_\pre), u, a) \sum_{u'_\pre \in \mathcal U^\pre} \Pr(U^\pre_{t+1} = u'_\pre | S_{t+1} = s') \Pr(S_{t+1} = s' | S_t = s, A_t = a),
\end{align}
where \textbf{(a)} follows from marginalization over $a$ and the definition of $\pi_i^\post$, \textbf{(b)} follows from marginalization over $u'_\pre$, \textbf{(c)} follows from the fact that (abbreviating and leaving out the common givens) $\Pr(u'_\pre) \Pr(s' | u'_\pre) = \Pr(u'_\pre | s') \Pr(s')$, and \textbf{(d)} follows from $M$'s conditional independence properties (applied to the second and third terms).  Notice that the second and third terms above are equivalent to $P$ and $\pi^\pre_i$; plugging those in and rearranging:
\begin{align}
    \Pr(\Tilde{X}_{t+1}.s = s' | \Tilde{X}_t.s = s, \Tilde{X}_t.u_\text{\emph{pre}} =& \sum_{u'_\pre \in \mathcal U^\pre} \pi^\pre_i(s', u'_\pre)\sum_{a \in \mathcal A} P(s, a,s') \pi^\post_i((s, u_\pre),u,a)\\
    \overset{\text{(a)}}{=}& \sum_{u'_\pre \in \mathcal U^\pre} P^i((s, u_\pre),u,(s', u_\pre'))\\
    \overset{\text{(b)}}{=}& \sum_{u'_\pre \in \mathcal U^\pre} \Pr(\Tilde X_{t+1}.u_\pre = u_\pre' | \Tilde X_t = (s, u_\pre), \Tilde U_t = u)\\
    &\times \Pr(\Tilde X_{t+1}.s = s'| \Tilde X_{t+1}.u_\pre = u_\pre', \Tilde X_t = (s, u_\pre), \Tilde U_t = u)\\
    \overset{\text{(c)}}{=}& \Pr(\Tilde{X}_{t+1}.s = s' | \Tilde{X}_t.s = s, \Tilde{X}_t.u_\pre = u_\pre, \Tilde{U}_t = u),
    \end{align}
    where \textbf{(a)} follows from the definition of $P^i$ for the CoMDP, \textbf{(b)} follows from the definition of conditional probability, and \textbf{(c)} follows from marginalization over $u'_\pre$.
\end{proof}

\begin{prop}
\label{u'_pre_independence}
\begin{align*}
    &\forall s \in \mathcal{S} \,\, \forall s' \in \mathcal{S} \,\, \forall u_\text{\emph{pre}} \in \mathcal U^\text{\emph{pre}}_i \,\, \forall u'_\text{\emph{pre}} \in \mathcal U^\text{\emph{pre}}_i \,\, \forall u \in \mathcal U, \,\\
    &\Pr(\tilde X_{t+1}.u_\text{\emph{pre}} = u'_\text{\emph{pre}} | \tilde X_{t+1}.s = s')
    = \Pr(\tilde X_{t+1}.u_\text{\emph{pre}} = u'_\text{\emph{pre}} | \tilde X_{t+1}.s = s', \tilde X_{t} = (s, u_\text{\emph{pre}}), \tilde U_t = u).
\end{align*}
\end{prop}
\begin{proof}
\begin{align}
    \Pr(&\tilde X_{t+1}.u_\pre = u'_\pre | \tilde X_{t+1}.s = s', \tilde X_{t} = (s, u_\pre), \tilde U_t = u)\\
    =& \frac{\Pr(\tilde X_{t+1}.u_\pre = u'_\pre, \tilde X_{t+1}.s = s' | \tilde X_{t} = (s, u_\pre), \tilde U_t = u)}{\Pr(\tilde X_{t+1}.s = s' | \tilde X_{t} = (s, u_\pre), \tilde U_t = u)}\\
    \overset{\text{(a)}}{=}& \frac{P^i((s, u_\pre), u, (s', u_\pre'))}{\Pr(S_{t+1} = s' | S_t = s, U^\pre = u_\pre, U_t = u)}\\
    =& \frac{\pi^\pre_i(s', u'_\pre)\sum_{a \in \mathcal A} P(s, a, s') \pi^\post_i((s, u_\pre),u,a)}{\Pr(S_{t+1} = s' | S_t = s, U^\pre = u_\pre, U_t = u)},
\end{align}
    where \textbf{(a)} follows from Property \ref{Phil_property} applied to the denominator.  Expanding the $P$ term and applying $M$'s conditional independence properties:
\begin{align}
    \Pr(&\tilde X_{t+1}.u_\pre = u'_\pre | \tilde X_{t+1}.s = s', \tilde X_{t} = (s, u_\pre), \tilde U_t = u)\\
    =& \frac{\pi^\pre_i(s', u'_\pre)\sum_{a \in \mathcal A} \Pr(S_{t+1} = s' | S_t = s, U^\pre_t = u_\pre, U_t = u, A_t = a) \pi^\post_i((s, u_\pre),u,a)}{\Pr(S_{t+1} = s' | S_t = s, U^\pre = u_\pre, U_t = u)}\\
    \overset{\text{(a)}}{=}& \frac{\pi^\pre_i(s', u'_\pre) \Pr(S_{t+1} = s' | S_t = s, U^\pre_t = u_\pre, U_t = u)}{\Pr(S_{t+1} = s' | S_t = s, U^\pre = u_\pre, U_t = u)}\\
    =& \Pr(\tilde X_{t+1}.u_\pre = u'_\pre | \tilde X_{t+1}.s = s'),\\
\end{align}
where \textbf{(a)} follows from marginalization over $a$.
\end{proof}

\begin{prop}
\label{correct_R_t}
\begin{align*}
    \forall r \in \mathcal R, \, \Pr(R_t = r) =& \Pr(\Tilde{R}^i_t = r).
\end{align*}
\end{prop}
\begin{proof}
\begin{align*}
    \Pr(R_t = r) =& \sum_{s \in \mathcal S} \Pr(S_t = s) \sum_{u_\pre \in \mathcal U^\pre} \Pr(U^\pre_t = u_\pre | S_t = s) \sum_{u \in \mathcal U} \Pr(U_t = u| S_t = s, U^\pre_t = u_\pre)\\
    &\times \sum_{s' \in S} \Pr(S_{t+1} = s' | S_t = s, U^\pre_t = u_\pre, U_t = u)\\
    &\times \sum_{u'_\pre \in \mathcal U^\pre} \Pr(U^\pre_{t+1} = u'_\pre |  S_t = s, U^\pre_t = u_\pre, U_t = u, S_{t+1} = s')\\
    &\times \Pr(R_t = r | S_t = s, U^\pre_t = u_\pre, U_t = u, S_{t+1} = s', U^\pre_{t+1} = u'_\pre),
\end{align*}
by repeated marginalization.  Applying $M$'s conditional independence properties to the $\Pr(U^\pre_{t+1} ...)$ term:
\begin{align*}
    \Pr(R_t = r)
    =& \sum_{s \in \mathcal S} \Pr(S_t = s) \sum_{u_\pre \in \mathcal U^\pre} \Pr(U^\pre_t = u_\pre | S_t = s) \sum_{u \in \mathcal U} \Pr(U_t = u| S_t = s, U^\pre_t = u_\pre)\\
    &\times \sum_{s' \in S} \Pr(S_{t+1} = s' | S_t = s, U^\pre_t = u_\pre, U_t = u) \sum_{u'_\pre \in \mathcal U^\pre} \Pr(U^\pre_{t+1} = u'_\pre |  S_{t+1} = s')\\
    &\times \Pr(R_t = r | S_t = s, U^\pre_t = u_\pre, U_t = u, S_{t+1} = s', U^\pre_{t+1} = u'_\pre) \\
    \overset{\text{(a)}}{=}& \sum_{s \in \mathcal S} \Pr(\tilde X_t.s = s) \sum_{u_\pre \in \mathcal U^\pre} \Pr(\tilde X_t.u_\pre = u_\pre | \tilde X_t.s = s) \sum_{u \in \mathcal U} \Pr(\tilde U_t = u| \tilde X_t = (s, u_\pre))\\
    &\times \sum_{s' \in S} \Pr(\tilde X_{t+1}.s = s' | \tilde X_t = (s, u_\pre), \tilde U_t = u) \sum_{u'_\pre \in \mathcal U^\pre} \Pr(\tilde X_{t+1}.u_\pre = u'_\pre |  \tilde X_{t+1}.s = s')\\
    &\times \Pr(\tilde R^i_t = r | \tilde X_t = (s, u_\pre), \tilde U_t = u, \tilde X_{t+1} = (s', u'_\pre)),
\end{align*}
where \textbf{(a)} follows from properties that show various equivalences between the two MDP's. Specifically: Property \ref{correct_s} (first term), Property \ref{correct_pi_pre} (second and fifth terms), Assumption \ref{supp_agent_equivalence} (third term), Property \ref{Phil_property} (fourth term), and Property \ref{correct_reward} (last term). Next, we apply Property \ref{u'_pre_independence} to the fifth term:
\begin{align*}
    \Pr(R_t = r)
    =& \sum_{s \in \mathcal S} \Pr(\tilde X_t.s = s) \sum_{u_\pre \in \mathcal U^\pre} \Pr(\tilde X_t.u_\pre = u_\pre | \tilde X_t.s = s) \sum_{u \in \mathcal U} \Pr(\tilde U_t = u| \tilde X_t = (s, u_\pre))\\
    &\times \sum_{s' \in S} \Pr(\tilde X_{t+1}.s = s' | \tilde X_t = (s, u_\pre), \tilde U_t = u)\\
    &\times \sum_{u'_\pre \in \mathcal U^\pre} \Pr(\tilde X_{t+1}.u_\pre = u'_\pre |  \tilde X_{t+1}.s = s', \tilde X_t = (s, u_\pre), \tilde U_t = u)\\
    &\times \Pr(\tilde R^i_t = r | \tilde X_t = (s, u_\pre), \tilde U_t = u, \tilde X_{t+1} = (s', u'_\pre)) \\
    \overset{\text{(a)}}{=}& (1)(1)(1)(1)(1) \Pr(\tilde R^i_t = r) \\
    =& \Pr(\tilde R^i_t = r),
\end{align*}
where \textbf{(a)} follows from repeated marginalization.
\end{proof}

\begin{lemma}
\label{is_mdp}
    $M^i$ is a Markov decision process.
\end{lemma}
\begin{proof}
Having defined $\mathcal X^i$ as the state set, $\mathcal U^i$ as the action set, $\mathcal R^i$ as the reward set, $P^i$ as the transition function, $R^i$ as the reward function, $d^i_0$ as the initial state distribution, and $\gamma_i$ as the discount parameter, all that remains is to ensure that $P^i$, $R^i$, and $d^i_0$ satisfy their necessary requirements. 
That is, we must show that these functions are always non-negative and that $\forall x \in \mathcal X, \forall u \in \mathcal U, \sum_{x'\in\mathcal X}P^i(x,u,x')=1$, $\forall x \in \mathcal X, \forall u \in \mathcal U, \forall x' \in \mathcal X, \sum_{r \in \mathcal R^i}\mathcal R^i(x,u,x',r)=1$, and $\sum_{x\in \mathcal X} d^i_0(x)=1$.

The functions are always non-negative because each term in each definition is always non-negative.
Next, we show that the sum over the transition function is $1$:
\begin{align}
    \forall x \in \mathcal X, \forall u \in \mathcal U,\\
    \sum_{x'\in\mathcal X}P^i(x,u,x')
    \overset{\text{(a)}}{=}& \sum_{x'\in\mathcal X} \Pr(S_{t+1} = x'.s, U^\pre_{t+1} = x'.u_\pre | S_t = x.s, U_t^\pre = x.u_\pre, U_t = u)\\
    =& \sum_{x'.s \in\mathcal S} \sum_{x'.u_\pre \in\mathcal \mathcal U^\pre} \Pr(S_{t+1} = x'.s, U^\pre_{t+1} = x'.u_\pre | S_t = x.s, U_t^\pre = x.u_\pre, U_t = u)\\
    =& 1,
\end{align}
where \textbf{(a)} follows from Property \ref{correct_transition}.
Next, we show that the sum over the reward function is $1$:
\begin{align}
    \forall x \in \mathcal X, \,\, \forall u \in \mathcal U,& \,\, \forall x' \in \mathcal X,\\
    \sum_{r \in \mathcal R^i} R^i(x,u,x',r) \overset{\text{(a)}}{=}& \sum_{r \in \mathcal R} \Pr(R_t = r | S_t = x.s, U_t^\pre = x.u_\pre, U_t = u, S_{t+1} = x'.s, U^\pre_{t+1} = x'.u_\pre) \\
    =& 1,
\end{align}
where \textbf{(a)} follows from the fact that $\mathcal R^i\coloneqq\mathcal R$ and from Property \ref{correct_reward}.

Finally, we show that the sum of the initial state distribution is 1:
\begin{align}
    \sum_{x\in \mathcal X} d^i_0(x)
    \overset{\text{(a)}}{=}& \sum_{x\in \mathcal X} d_0(x.s) \underbrace{\pi^\pre_i(x.s,x.u_\pre)}_{=\Pr(U^\pre_0 = x.u_\pre | S_0 = x.s)}\\
    =&\sum_{x.s \in\mathcal S} \sum_{x.u_\pre \in\mathcal \mathcal U^\pre} \Pr(S_0 = x.s, U^\pre_0 = x.u_\pre)\\
    =& 1,
\end{align}
where \textbf{(a)} follows from the definition of $d^i_0$ for the CoMDP.

Therefore, $M^i$ is a Markov decision process.
\end{proof}

\begin{lemma}
\label{correct_comdp}
    For all $M,i,\pi_i^\text{\emph{pre}}, \pi_i^\text{\emph{post}}$, and $\bar \theta_i$, and given a policy parameterized by $\theta_i$, the corresponding CoMDP $M^i$ satisfies: 
    \begin{itemize}
        \item $ \forall x\in \mathcal X \,\, \forall x' \in \mathcal X \,\, \forall u \in \mathcal U  \,\,  \forall r \in \mathcal{R},  \,\, P^i(x,u,x') \\=\Pr(S_{t+1}=x'.s,U_{t+1}^\text{\emph{pre}}=x'.u_\text{\emph{pre}}|S_t=x.s, U_t^\text{\emph{pre}}=x.u_\text{\emph{pre}}, U_t=u)$.
        \item $\forall x\in \mathcal X \,\, \forall x' \in \mathcal X \,\, \forall u \in \mathcal U  \,\,  \forall r \in \mathcal{R},  \\ R^i(x,u,x',r) = \Pr(R_t{=}r|S_t{=}x.s,U_t^\text{\emph{pre}}{=}x.u_\text{\emph{pre}},U_t{=}u,S_{t+1}{=}x'.s, U_{t+1}^\text{\emph{pre}}{=}x'.u_\text{\emph{pre}})$.
        \item $\forall s \in \mathcal S \,\, \forall u_\text{\emph{pre}} \in \mathcal U^\text{\emph{pre}}, \,\, \Pr(S_t=s, U^\text{\emph{pre}}_t = u_\pre)=\Pr(\tilde X_t = (s, u_\pre))$.
        \item $\forall s \in \mathcal S, \,\, \Pr(S_t=s)=\Pr(\tilde X_t.s=s)$.
        \item $\forall r \in \mathcal R, \,\,\Pr(R_t=r)=\Pr(\tilde R_t^i=r)$.
    \end{itemize}
\end{lemma}

\begin{proof}
    This follows immediately from properties \ref{correct_transition}, \ref{correct_reward}, \ref{correct_x}, \ref{correct_s}, and \ref{correct_R_t}.
\end{proof}

\begin{prop}
\label{J_equivalence}
    For all coagents $i$, for all $\theta_i$, given the same $\theta = (\theta_i, \bar \theta_i)$, $J(\theta) = J_i(\theta_i)$.
\end{prop}
\begin{proof}
\begin{align*}
    J(\theta)=& \, \mathbf{E}\left [ \sum_{t=0}^\infty \gamma^t R_t | \theta \right]\\
    =&\, \mathbf{E}\left [ \sum_{t=0}^\infty \gamma_i^t R^i_t | \theta_i, \bar \theta_i, \right]\\
    =& J_i(\theta_i),
\end{align*}
where the second step follows directly from Property \ref{correct_R_t} and the definition of $\gamma_i$.
\end{proof}

\begin{lemma}
\label{delta_is_dJ}
    For all coagents $i$, for all $\theta_i$, $\frac{\partial J_i(\theta_i)}{\partial \theta_i} = \Delta_i(\theta_i).$
\end{lemma}
\begin{proof}
    In Lemma \ref{is_mdp}, we proved that the $\ith$ CoMDP is an MDP. In Lemma \ref{correct_comdp}, we proved that the $\ith$ CoMDP correctly models the $\ith$ coagent's environment. Lemma \ref{delta_is_dJ} follows directly from these results and the fact that $\Delta_i$ is the policy gradient for $M^i$ (Sutton, 2000).
\end{proof}

\begin{thm}
    $\\ \nabla J(\theta)=\left [\Delta_1(\theta_1)^\intercal , \Delta_2(\theta_2)^\intercal,\dotsc,\Delta_m(\theta_m)^\intercal\right ]^\intercal$, 
    where $m$ is the number of coagents, and $\Delta_i$ is the local policy gradient of the $i$th coagent.
\end{thm}
\begin{proof}
\begin{align}
    \nabla J(\theta) =&
    \left [ \frac{\partial J(\theta)}{\partial \theta_1}^\intercal , \frac{\partial J(\theta)}{\partial \theta_2}^\intercal,\dotsc,\frac{\partial J(\theta)}{\partial \theta_m}^\intercal\right ]^\intercal\\
    \overset{\text{(a)}}{=}& \left [ \frac{\partial J_1(\theta_1)}{\partial \theta_1}^\intercal , \frac{\partial J_2(\theta_2)}{\partial \theta_2}^\intercal,\dotsc,\frac{\partial J_m(\theta_m)}{\partial \theta_m}^\intercal\right ]^\intercal\\
    \overset{\text{(b)}}{=}& \left [ \frac{\Delta_1(\theta_1)}{\partial \theta_1}^\intercal , \frac{\Delta_2(\theta_2)}{\partial \theta_2}^\intercal,\dotsc,\frac{\Delta_m(\theta_m)}{\partial \theta_m}^\intercal\right ]^\intercal,
\end{align}
where \textbf{(a)} follows directly from Property \ref{J_equivalence} and where \textbf{(b)} follows directly from Lemma 3.
\end{proof}

\begin{cor}
\label{cpgt_algo}
    If $\alpha_t$ is a deterministic positive stepsize, $\sum_{t=0}^\infty \alpha_t = \infty$, $\sum_{t=0}^\infty \alpha_t^2 < \infty$, additional technical assumptions are met 
    (Bertsekas $\&$ Tsitsiklis, 2000, Proposition 3), and each coagent updates its parameters, $\theta_i$, with an unbiased local policy gradient update $\theta_i \gets \theta_i + \alpha_t \widehat \Delta_i(\theta_i)$, then $J(\theta)$ converges to a finite value and $\lim_{t \to \infty} \nabla J(\theta) = 0$. 
\end{cor}
\begin{proof}
     Corollary \ref{cpgt_algo} follows directly from the CPGT, Proposition 3 from Bertsekas \& Tsitsiklis (2000),
     and the assumption that the discounted sum of rewards over an episode is finite (this last assumption prevents $J(\theta)$ from diverging to $\infty$).
\end{proof}

\section{Asynchronous Coagent Networks: Supplementary Proofs}
\label{sec:async_supp}

\subsection{Synchronous Network Correctness}
\label{sec:sync_network}

Our goal is to show that the synchronous, acyclic reduction of our original asynchronous, cyclic network behaves identically to our original network.
%
That is, for all $s \in \mathcal S, u \in \mathcal U^\text{all}, a \in \mathcal A, e \in \{0, 1\}^m$, $\grave \pi((s, u), (a, u', e)) = \Pr(A_t = a, U^\all_{t} = u', E_t = e | S_t = s, U^\all_{t-1} = u)$.
Because of the large number of variables, if we use one of these lowercase symbols in an equation, assume that it holds for all values in its respective set.

\begin{proof}

We present a proof by induction.
We assume a topological ordering of the coagents, such that for any $j < i$, the $j^\text{th}$ coagent executes before the $\ith$ coagent.
We perform induction over $i$, with the inductive assumption that the outputs of all the previous coagents, as well as their decisions whether or not to execute, correspond to the original network.
The inductive hypothesis is that for all $j < i$:
%
\begin{equation}
\Pr(\grave A_t.u^\all_j = u'_j, \grave A_t.e_j = e_j | \grave S_t = (s, u)) =\Pr(U^\all_{t}.u_j = u'_j, E^j_t = e_j | S_t = s, U^\all_{t-1} = u).
    \label{eq:reduced_network_inductive_assumption}
\end{equation}
Consider the base case, $i=1$.
$\grave{\mathcal U_1^\pre}$ and $\mathcal U_1^\pre$ are both the empty set, because no coagents produce an output before the first coagent in either network.
As a result, the distribution over the execution probability is trivially the same in both networks, that is, 
$\Pr(E^1_t = 1 | S_t = s, U^\all_{t-1} = u) = \beta_1((s, \varnothing, u)) = \Pr(\grave A_t.e_1 = 1 | \grave S_t = (s, u))$.
Next, we consider the action. 
%
If the coagent executes, $\Pr(U^1_t = u'_1 | E^1_t = 1, S_t = s, U^\all_{t-1} = u) = \pi_1((s, \varnothing, u), u'_1) = \Pr(\grave A_t.u^\all_1 = u'_1 | A_t.e_1 = 1, \grave S_t = (s, u))$.
If the coagent does not execute, the action is trivially $u_1$ in both cases.
Therefore, Equation \eqref{eq:reduced_network_inductive_assumption} holds for $j=1$.

Next we consider the inductive step, where we show that Equation \eqref{eq:reduced_network_inductive_assumption} holds for the $\ith$ coagent given that it holds for $j < i$.
First we consider the execution function, the output of which is represented in the synchronous setting by $\grave A_t.e_i$, and in the asynchronous setting by $E_t$.  
In the asynchronous setting, the probability of the $i^\text{th}$ coagent executing at time step $t$ is 
$\beta_i((S_t, U^\pre_{t}, U^\all_{t-1}))$.
Since we are not given $U^\pre_t$, we must sum over possible values:
%
\begin{equation}
    \Pr(E^i_t = 1 | S_t = s, U^\all_t = u) = 
    \sum_{u_\pre \in \mathcal U^\pre_i} \beta_i((s, u_\pre, u)) \Pr(U^\pre_{t} = u_\pre | S_t = s, U^\all_{t-1} = u).
\end{equation}
%
In the reduced setting, we instead have a coagent, such that $\Pr(\grave A_t.e_i = 1 | \grave S_t = (s, u)) = \beta_i((s, \grave U^\pre_{t}, u))$.
Again, we sum over possible values of $\grave U^\pre_{t}$:
%
\begin{align*}
    \Pr(\grave A_t.e_i = 1 | \grave S_t = (s, u))
    &= \sum_{u_\pre \in \mathcal U^\pre_i} \beta_i((s, u_\pre, u)) \Pr(\grave U^\pre_{t} = u_\pre | \grave S_t = (s, u)).
\end{align*}
Recall the reduced setting was defined such that for all $j < i, \grave A_t.u^\all_j = \grave U^\pre_{t}.u_j$, and in the asynchronous setting, $U^\pre_{t}.u_j = U^\all_{t}.u_j$.
%
We therefore can conclude from \eqref{eq:reduced_network_inductive_assumption} and by substitution that for all $j < i, \Pr(\grave U^\pre_{t}.u_j = u | \grave S_t = (s, u)) =\Pr(U^\pre_{t}.u_j = u | S_t = s, U^\all_{t-1} = u)$.
Substituting this into the above equations:

\begin{align*}
    \Pr(\grave A_t.e_i = 1 | \grave S_t = (s, u))
    &= \sum_{u_\pre \in \mathcal U^\pre_i} \beta_i((s, u_\pre, u)) \Pr(U^\pre_{t} = u_\pre | S_t = s, U^\all_{t-1} = u) \\
    &= \Pr(E^i_t = 1 | S_t = s, U^\all_t = u).
\end{align*}
Note also that from the perspective of $\grave \pi_i$, $\grave A_t.e_i = \grave U^\pre_t.e_i$.
Next we consider the output of the $\ith$ coagent, given in the asynchronous setting as $U^i_t$, and in the reduced setting by $\grave A_t.u^\all_i$.
In the original setting, $U^i_t$ was given such that for all $u_i \in U^i$:
%
\begin{equation}
    \Pr(U^i_t = u'_i | S_t = s, U^\all_{t-1} = u, U^\pre_t = u_\pre, E^i_t = e_i) = 
    \begin{cases}
        \pi_i((s, u_\pre, u), u'_i), &\text{if } e_i = 1 \\
        1, &\text{if } e_i = 0 \text{ and } u'_i = u_i \\
        0, &\text{otherwise}.
    \end{cases}
\end{equation}
In the synchronous setting, we are given:
%
\begin{align*}
    \Pr(\grave A_t.u^\all_i = u'_i | \grave S_t = (s, u), \grave U^\pre_t.u = u_\pre, \grave U^\pre_t.e_i = e_i) 
    &= \grave \pi_i(((s, u), u_\pre), u) \\
    &= \begin{cases}
        \pi_i((s, u_\pre, u), u'_i), &\text{if } e_i = 1 \\
        1, &\text{if } e_i = 0 \text{ and } u'_i = u_i \\
        0, &\text{otherwise}.
    \end{cases}
\end{align*}
%
Since we were given $s$ and $u$, assumed through the inductive hypothesis that $\Pr(\grave U^\pre_{t}.u = u_\pre | \grave S_t = (s, u)) = \Pr(U_t^\pre = u_\pre | S_t = s, U^\all_{t-1} = u)$, and showed that $\Pr(\grave U^\pre_{t}.e_i = e_i| \grave S_t = (s, u)) = \Pr(E^i_t = e_i | S_t = s, U^\all_{t-1} = u)$,
we know that the distributions over the variables we conditioned on are equal.
%
Since we also showed that the conditional distributions are equal, we conclude that $\Pr(\grave A_t.u^\all_i = u'_i | \grave S_t = (s, u)) = \Pr(U^i_t = u'_i | S_t = s, U^\all_{t-1} = u)$.

%
This completes the inductive proof that $\Pr(\grave A_t.u^\all = u', \grave A_t.e = e | \grave S_t = (s, u)) = \Pr(U^\all_t = u', E_t = e | S_t = s, U^\all_{t-1} = u)$.
We still must consider $\grave A_t.a$.
This is given by the output of some predefined subset of coagents, which is the same subset in both the synchronous and asynchronous network.
%
%
We showed that the distribution over outputs was the same for corresponding coagents in the two networks, and therefore can conclude immediately that $Pr(\grave A_t.a = a | \grave S_t = (s, u)) = \Pr(A_t = a | S_t = s, U^\all_{t-1} = u)$. Finally:
%
\begin{align*}
    \Pr(A_t = a, U^\all_t = u', E_t = e | S_t = s, U^\all_{t-1} = u) 
    &= \Pr(\grave A_t.a = a, \grave A_t.u^\all = u', \grave A_t.e = e | \grave S_t = (s, u)) \\
    &= \Pr(\grave A_t = (a, u', e) | \grave S_t = (s, u)) \\
    &= \grave \pi((s, u), (a, u', e)).
\end{align*}

\end{proof}

\subsection{Equivalence of Objectives}

In both settings, the network depends on the same parameter vector, $\theta$.
In this section, we show that for all settings of this parameter vector the resulting sum of rewards is equivalent in both settings.
That is,
$J(\theta) = \grave J(\theta)$.

\begin{proof}
    We begin by showing that the distribution over the ``true" states and actions is equal in both settings, that is, for all $s \in \mathcal S, a \in \mathcal A$, $Pr(\grave S_t.s = s, \grave A_t.a = a) =\Pr(S_t = s, A_t = a)$.
    Once this is shown, we show that the reward distributions are the same, that is, for all $r$, $\Pr(\grave R_t = r) = \Pr(R_t = r)$.
    Finally, we show $J(\theta) = \grave J(\theta)$.
    
    \subsubsection{Equivalence of State Distributions}
    
    First, we show that $\Pr(\grave S_t = (s, u)) = \Pr(S_t = s, U^\all_{t - 1} = u)$, by induction over time steps.
    The base case is the initial state, $\grave S_0$. 
    %
    %
    We assumed in the problem setup for the asynchronous setting that for all $i$ and $j$, the random variables $S_0$, $U^i_{-1}$, and $U^j_{-1}$ are independent.
    %
    For all $s$ and $u$:
    \begin{align*}
        \Pr(\grave S_0 = (s, u)) 
        &= \grave d_0((s,u)) \\
        &= d_0(s) \prod_{i = 1}^m h^i_0(u_i) \\
        &= \Pr(S_0 = s) \prod_{i = 1}^m \Pr(U^i_{-1} = u_i) \\
        &= \Pr(S_0 = s, U^\all_{-1} = u) \\
        &= \Pr(S_0 = s, U^\all_{-1} = u).
    \end{align*}
    %
    %
    Thus, we've proven the base case. Next we consider the inductive step: 
    %
    \begin{align*}
        &\Pr(\grave S_{t+1} = (s', u') | \grave S_t = (s, u)) \\
        &= \sum_{(a, u'', e) \in \grave A} \Pr(\grave S_{t+1} = (s', u') | \grave S_t = (s, u), \grave A_t = (a, u'', e)) \Pr(\grave A_t = (a, u'', e) | \grave S_t = (s, u)) \\
        &= \sum_{(a, u'', e) \in \grave A} \grave P((s, u), (a, u'', e), (s', u')) \grave \pi((s, u), (a, u'', e)) \\
        &= \sum_{a \in \mathcal A, u'' \in \mathcal U^\all, e \in \mathcal E} 
        \begin{cases}
            P(s, a, s') \grave \pi((s,u), (a, u'', e)) &\text{if } u' = u'' \\
            0 &\text{otherwise},
        \end{cases}
    \end{align*}
    The case statement comes from the definition of $\grave P$.
    Clearly, we can eliminate all of the parts of the summation where $u' \neq u''$.
    Therefore:
    \begin{align*}
        \Pr(\grave S_{t+1} = (s', u') | \grave S_t = (s, u)) = \sum_{a \in \mathcal A, e \in \mathcal E} P(s, a, s') \grave \pi((s, u), (a, u', e)).
    \end{align*}
    Next, we can apply the equivalence shown in section \ref{sec:sync_network} and the definition of $P$:
    \begin{align*}
        &\sum_{a \in \mathcal A, e \in \mathcal E} P(s, a, s') \grave \pi((s, u), (a, u', e)) \\
        = &\sum_{a \in \mathcal A, e \in \mathcal E} \Pr(S_{t+1} = s' | A_t = a, S_t = s) \Pr(A_t = a, U^\all_t = u', E_t = e | S_t = s, U^\all_{t-1} = u). \\
    \end{align*}
    By the law of total probability, we can eliminate $E_t$ from the expression:
    \begin{align*}
        \sum_{a \in \mathcal A} \Pr(S_{t+1} = s' | A_t = a, S_t = s) \Pr(A_t = a, U^\all_t = u' | S_t = s, U^\all_{t-1} = u). \\
    \end{align*}
    Next, $S_{t+1}$ is conditionally independent of $U_t^\all$ and $U_{t-1}^\all$ given $S_t$ and $A_t$, so we can rewrite the expression as a sum over a single probability:
    \begin{align*}
        &\sum_{a \in \mathcal A} \Pr(S_{t+1} = s' | A_t = a, U^\all_t = u', S_t = s, U^\all_{t-1} = u) \Pr(A_t = a, U^\all_t = u'| S_t = s, U^\all_{t-1} = u) \\
        = &\sum_{a \in \mathcal A} \Pr(S_{t+1} = s', A_t = a, U^\all_t = u' | S_t = s, U^\all_{t-1} = u).
    \end{align*}
    Finally, by the law of total probability, we eliminate $A_t$:
    \begin{align*}
        &\sum_{a \in \mathcal A} \Pr(S_{t+1} = s', A_t = a, U^\all_t = u' | S_t = s, U^\all_{t-1} = u) = \Pr(S_{t+1} = s', U^\all_t = u' | S_t = s, U^\all_{t-1} = u).
    \end{align*}
    Thus, the inductive hypothesis holds.
    We have therefore shown that for all $t$, $\Pr(\grave S_t = (s, u)) = \Pr(S_t = s, U^\all_{t-1} = u)$.

    \subsubsection{Equivalence of Reward Distributions}
    
    It follows immediately from the above equality and \ref{sec:sync_network} that $\Pr(\grave A_t = (a, u, e)) = \Pr(A_t = a, U^\all_t = u, E_t = e)$.
    We turn our attention to the reward distribution:
    
    
    \begin{align*}
        &\Pr(\grave R_t = r)  \\
        = & \sum_{(s, u) \in \grave{\mathcal S}} \sum_{(a, u', e) \in \grave{\mathcal A}} \sum_{(s', u'') \in \grave{\mathcal S}} \Pr(\grave R_t = r | \grave S_t = (s, u), \grave A_t = (a, u', e), \grave S_{t + 1} = (s', u'')) \\
        & \times \Pr(\grave S_t = (s, u), \grave A_t = (a, u', e), \grave S_{t + 1} = (s', u'')) \\
        =& \sum_{(s, u) \in \grave{\mathcal S}} \sum_{(a, u', e) \in \grave{\mathcal A}} \sum_{(s', u'') \in \grave{\mathcal S}} \grave{R}((s,u), (a, u', e), (s', u'')) \Pr(\grave S_t = (s, u), \grave A_t = (a, u', e), \grave S_{t + 1} = (s', u'')) \\
        =& \sum_{(s, u) \in \grave{\mathcal S}} \sum_{(a, u', e) \in \grave{\mathcal A}} \sum_{(s', u'') \in \grave{\mathcal S}} R(s, a, s') \Pr(\grave S_t = (s, u), \grave A_t = (a, u', e), \grave S_{t + 1} = (s', u'')) \\
        =& \sum_{(s, u) \in \grave{\mathcal S}} \sum_{(a, u', e) \in \grave{\mathcal A}} \sum_{(s', u'') \in \grave{\mathcal S}} R(s, a, s') \Pr(S_t = s, U^\all_{t-1} = u, A_t = a, U^\all_t = u', E_t = e, S_{t + 1} = s') \\
       =& \sum_{s \in \mathcal S} \sum_{a \in \mathcal A} \sum_{s' \in \mathcal S} R(s, a, s') \Pr(S_t = s, A_t = a, S_{t+1} = s') \\
       =& \sum_{s \in \mathcal S} \sum_{a \in \mathcal A} \sum_{s' \in \mathcal S} \Pr(R_t = r | S_t = s, A_t = a, S_{t+1} = s') \Pr(S_t = s, A_t = a, S_{t+1} = s') \\
       =& \Pr(R_t = r).
    \end{align*}
    
    \subsubsection{Equivalence of Objectives}
    
    Finally, we show the objectives are equal: 
    \begin{align*}
        \grave J(\theta) &= \mathbf E[\sum_{t=0}^T \grave \gamma^t \grave R_t]
        = \mathbf E[\sum_{t=0}^T \gamma^t R_t]
        = J(\theta),
    \end{align*}
    by linearity of expectation.
\end{proof}

\section{Experimental Details of Finite Difference Comparison}
\label{sec:finite_dif_details}

\begin{figure}
  \centering
  \includegraphics[width=.5\linewidth]{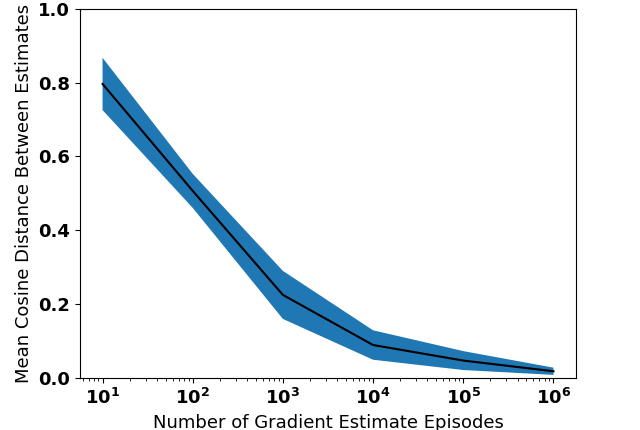}
  \caption{The average (across coagents) cosine distance between the gradient estimates of the finite difference method and the CPGT (vertical axis) versus the number of episodes used for the CPGT estimate (horizontal axis). This data is drawn from $20$ trials. Error bars represent standard error. $5\times10^{8}$ episodes are used for each finite difference estimate. As the amount of data used for the CPGT gradient estimate increases, the cosine distance approaches zero, indicating that the two gradient estimates converge to the same value as the amount of data increases.}
  \label{finite_dif}
\end{figure}

To empirically test the Asynchronous Coagent Policy Gradient Theorem (ACPGT), we compare the gradient ($\nabla J$) estimates of the ACPGT and a finite difference method.
Finite difference methods are a well-established technique for computing the gradient of a function from samples; they serve as a straightforward baseline to evaluate the gradients produced by our algorithm.
We expect these estimates to approach the same value as the amount of data used approaches infinity. For the purposes of testing the ACPGT, we use a simple toy problem and an asynchronous coagent network. The results are presented in Figure \ref{finite_dif}; this data provides empirical support for the ACPGT.

We use a simple $3 \times 3$ Gridworld.
The network structure used in this experiment consists of three coagents with tabular state-action value functions and softmax policies: Two coagents receive the tabular state as input, and each of those two coagents have a single tabular binary output to the third coagent, which in turn outputs the action (up, down, left, or right).  This results in two coagents with $18$ parameters each, and one coagent with $16$ parameters, resulting in a network with $52$ parameters.
The coagents asynchronously execute using a geometric distribution; the environment updates every step and each coagent has a 0.5 probability of executing each step. The gradient estimates appear to converge, providing empirical support of the CPGT.
The data is drawn from $20$ trials. $5\times10^{8}$ episodes were used for each finite difference estimate.
For each trial, five training episodes were conducted before the parameters were frozen and the two gradient estimation methods were run.
The coagents were trained with Sutton \& Barto's (2018) actor-critic with eligibility traces algorithm and shared a single critic. Note that the critic played no role in the gradient estimation methods, only in the initial training episodes.
Hyperparmaters used:
critic step size $= 0.024686$,
$\gamma = 1$,
input agent step size $= 0.02842$,
output agent step size $= 0.1598$, and
all agents' $\lambda = 0.8085$.

\begin{figure}
  \centering
  \includegraphics[width=.5\linewidth]{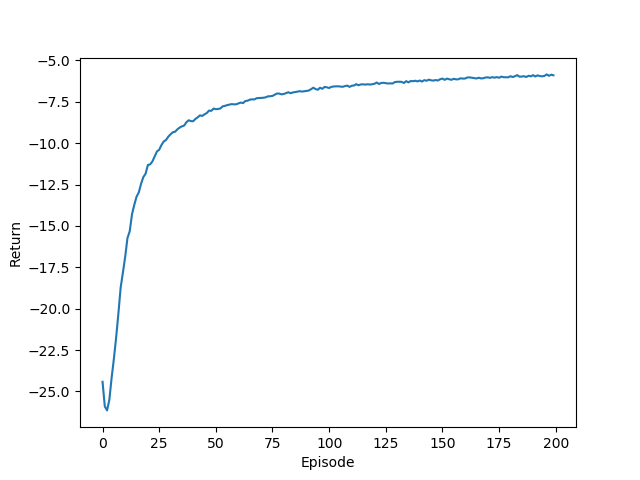}
  \caption{The mean learning curve from 10,000 trials of coagent network described in Section \ref{sec:finite_dif_details} (without freezing the parameters after 5 episodes). While a large amount of data is required to reduce the cosine distance to near 0, note that this fact does not reflect how long the network takes to learn near-optimal behavior.}
  \label{fig:coagents_convergence}
\end{figure}

Note that, while a large amount of data is required to reduce the cosine distance to near 0, this does not reflect how long
the network takes to learn near-optimal behavior. Figure \ref{fig:coagents_convergence} depicts the mean episodic return of 10,000 trials of 200
episodes each (the same environment, algorithms, network structure, hyperparameters, etc. described above).
Despite its handicap of only having a coagent execute with a 0.5 probability at each time step (a rather significant
handicap for this network structure in a gridworld), the network achieves near-optimal returns relatively quickly.

\section{Option-Critic}

\subsection{Option-Critic Complete Description}

\label{sec:option-critic-complete-description}

In this section, we adhere mostly to the notation given by Bacon et al. (2017)'s, with some minor changes used to enhance conceptual clarity regarding the inputs and outputs of each policy.
In the option-critic framework, the agent is given a set of \emph{options}, $\Omega$.
The agent selects an option, $\omega {\in} \Omega$, by sampling from a policy $\pi_\Omega: \mathcal S \times \Omega \to [0, 1]$.
An action, $a {\in} \mathcal A$, is then selected from a policy which considers both the state and the current option: $\pi_\omega: (\mathcal S \times \Omega) \times \mathcal A \to [0, 1]$.
A new option is not selected at every time step;
rather, an option is run until a \emph{termination function}, $\beta: (\mathcal S \times \Omega) \times \{0, 1\} \to [0, 1]$, selects the termination action, 0.
If the action 1 is selected, then the current option continues.
$\pi_\omega$ is parameterized by weights $\theta$, and $\beta$ by weights $\vartheta$.

\subsection{Option-Critic Gradient Equivalence}
\label{sec:option-critic-equiv}

The APCGN expression gives us 
$
    \frac{\partial J}{\partial \vartheta} = \sum_{x \in (\mathcal S \times \Omega)} d^\pi_\Omega(x) \sum_{u \in \{0, 1\}} \frac{\partial \beta(x, u)}{\partial \vartheta} Q_\beta(x, u).
$
We will show that this is equivalent to Bacon et al. (2017)'s expression for $\frac{\partial J}{\partial \vartheta}$.
Note that we only have two actions, whose probabilities must sum to one.
Therefore, the gradients of the policy are equal in magnitude but opposite in sign.
That is, for all $x \in (\mathcal S \times \Omega)$: 
$\beta(x, 0) + \beta(x, 1) = 1,$
so 
$\frac{\partial \beta(x, 0)}{\partial \vartheta} = -\frac{\partial \beta(x, 1)}{\partial \vartheta}.$
Additionally, we know that $Q_\beta(x, 1)$ is the expected value of continuing option $x.\omega$ in state $x.s$, given by $Q_\Omega(x.s, x.\omega)$.
$Q_\beta(x, 0)$ is the expected value of choosing a new action in state $x.s$, given by $V_\Omega(x.s)$,
and therefore, $Q_\beta(x, 1) - Q_\beta(x, 0) = A_\Omega(x.s, x.\omega)$.
The full derivation is:
\begin{align*}
    \frac{\partial J}{\partial \vartheta} &{=} \!\!\!\!\!\sum_{x \in (\mathcal S \times \Omega)} \!\!\!\!\!\! d^\pi_\Omega(x) \bigg[\frac{\partial \beta(x, 0)}{\partial \vartheta} Q_\beta(x, 0) {+} \frac{\partial \beta(x, 1)}{\partial \vartheta} Q_\beta(x, 1)\bigg] \\
    &{=} -\!\!\!\!\!\sum_{x \in (\mathcal S \times \Omega)} d^\pi_\Omega(x) \frac{\partial \beta(x, 0)}{\partial \vartheta} (Q_\beta(x, 1) - Q_\beta(x, 0)) \\
    &{=} -\!\!\!\!\!\sum_{x \in (\mathcal S \times \Omega)} d^\pi_\Omega(x) \frac{\partial \beta(x, 0)}{\partial \vartheta} A_\Omega(x.s, x.\omega).
\end{align*}
We see that the result is exactly equivalent to the expression for $\partial J / \partial \vartheta$ derived by Bacon et al. (2017).

\end{document}